\documentclass{article}
\usepackage[utf8]{inputenc}

\usepackage[margin=1in]{geometry}

\usepackage{microtype}
\usepackage{graphicx}
\usepackage{subfigure}
\usepackage{booktabs} %

\usepackage{amsmath}

\usepackage[colorlinks=true,allcolors=blue,hyperfootnotes=false]{hyperref}
\usepackage[capitalise]{cleveref}

\newenvironment{ack}{\subsection*{Acknowledgements}}

\title{Online Policy Gradient for Model Free Learning \\ of Linear Quadratic Regulators with $\sqrt{T}$ Regret}

\author{%
Asaf Cassel%
\thanks{School of Computer Science, Tel Aviv University; \texttt{acassel@mail.tau.ac.il}.}
\and
Tomer Koren%
\thanks{School of Computer Science, Tel Aviv University, and Google Research, Tel Aviv; \texttt{tkoren@tauex.tau.ac.il}.}
}

\RequirePackage{mathtools}
\RequirePackage{dsfont}
\RequirePackage{delimset}

\newcommand{\indEvent}[2][*]{\mathds{1}{\brk[c]#1{#2}}}

\newcommand{\tr}[2][*]{\mathrm{Tr}\brk*{#2}}

\newcommand{\tran}{^{\mkern-1.5mu\mathsf{T}}}

\newcommand{\Oof}[2][*]{O\brk#1{#2}}
\newcommand{\OtildeOf}[2][*]{\smash{\widetilde{O}}\brk#1{#2}}
\newcommand{\EE}[1][]{\mathbb{E}_{#1}}
\newcommand{\EEBrk}[2][*]{\EE\delim{[}{]}#1{#2}}
\newcommand{\RR}[1][]{\mathbb{R}^{#1}}

\newcommand{\PP}[2][*]{\mathbb{P}\brk#1{#2}}

\newcommand{\sphere}[1]{\mathcal{S}^{#1}}
\newcommand{\ball}[1]{\mathcal{B}^{#1}}

\newcommand{\gaussDist}[2]{\mathcal{N}\brk0{#1, #2}}

\DeclareMathOperator*{\argmin}{arg\,min}

\DeclarePairedDelimiterX\setDef[1]\lbrace\rbrace{#1}

\newcommand{\ifrac}[2]{#1/#2}
\newcommand{\minEigOf}[2][*]{\lambda_{\mathrm{min}}\brk#1{#2}}
\newcommand{\maxEigOf}[2][*]{\lambda_{\mathrm{max}}\brk#1{#2}} %
\usepackage{amsthm}
\usepackage{amssymb}
\usepackage{thmtools}

\declaretheoremstyle[
	    spaceabove=\topsep, 
	    spacebelow=\topsep, 
	    bodyfont=\normalfont\itshape,
    ]{theorem}

\declaretheorem[style=theorem,name=Theorem]{theorem}

\declaretheoremstyle[
	    spaceabove=\topsep, 
	    spacebelow=\topsep, 
	    bodyfont=\normalfont,
    ]{definition}

\declaretheoremstyle[
        spaceabove=\topsep, 
        spacebelow=\topsep, 
        bodyfont=\normalfont,
        notefont=\normalfont\bfseries,
        notebraces={}{},
        qed=$\blacksquare$, 
    ]{proofstyle}
\declaretheorem[style=proofstyle,numbered=no,name=Proof]{proof}

\declaretheorem[style=theorem,name=Lemma]{lemma}

\declaretheorem[style=theorem,name=Proposition]{proposition}

\declaretheorem[style=theorem,numbered=no,name=Theorem]{theorem*}
\declaretheorem[style=theorem,numbered=no,name=Lemma]{lemma*}
\declaretheorem[style=theorem,numbered=no,name=Corollary]{corollary*}
\declaretheorem[style=theorem,numbered=no,name=Proposition]{proposition*}
\declaretheorem[style=theorem,numbered=no,name=Claim]{claim*}
\declaretheorem[style=theorem,numbered=no,name=Fact]{fact*}
\declaretheorem[style=theorem,numbered=no,name=Observation]{observation*}
\declaretheorem[style=theorem,numbered=no,name=Conjecture]{conjecture*}

\declaretheorem[style=definition,name=Definition]{definition}

\declaretheorem[style=definition,numbered=no,name=Definition]{definition*}
\declaretheorem[style=definition,numbered=no,name=Remark]{remark*}
\declaretheorem[style=definition,numbered=no,name=Example]{example*}
\declaretheorem[style=definition,numbered=no,name=Question]{question*}
\declaretheorem[style=definition,numbered=no,name=Assumption]{assumption*}

\newcommand{\fDomain}{\mathcal{X}}
\newcommand{\costLevel}{\bar{f}}
\newcommand{\costMin}{f_*}
\newcommand{\cost}{f}
\newcommand{\costOf}[1]{f\brk*{#1}}

\newcommand{\gradOracle}[1][t]{\hat{g}_{#1}}

\newcommand{\gradCorrpution}[1][t]{w_{#1}}

\newcommand{\pSmooth}{\beta}

\newcommand{\pPL}{\mu}
\newcommand{\pRadius}{D_{0}}
\newcommand{\pCorrupt}[1][t]{\varepsilon_{#1}}
\newcommand{\pEffCorrupt}[1][t]{\bar{\varepsilon}_{#1}}

\newcommand{\pLip}{G}

\newcommand{\costMatLower}{\alpha_0}

\newcommand{\noiseMinStd}{\sigma}
\newcommand{\systemBound}{\psi}

\newcommand{\noiseBound}{W}

\newcommand{\Astar}{A_{\star}}
\newcommand{\Bstar}{B_{\star}}
\newcommand{\Kstar}{K_{\star}}

\newcommand{\Jstar}{J_{\star}}

\newcommand{\J}{J}
\newcommand{\Jof}[2][*]{\J\brk#1{#2}}
\newcommand{\regret}[1][T]{R_{#1}}

\newcommand{\ct}[1][t]{c_{#1}}

\newcommand{\dCi}[1][i]{\Delta C_{#1}}

\newcommand{\cjis}[1][\ind0,\ind1,\ind2]{c_{#1}}
\newcommand{\cjitau}[1][\ind0,\ind1,\tau]{c_{#1}}
\newcommand{\uji}[1][\ind0,\ind1]{U_{#1}}
\newcommand{\mj}[1][\ind0]{m_{#1}}
\newcommand{\rj}[1][\ind0]{r_{#1}}
\newcommand{\gj}[1][\ind0]{\hat{g}_{#1}}

\newcommand{\Kj}[1][\ind0]{K_{#1}}
\newcommand{\Kji}[1][\ind0,\ind1]{K_{#1}}

\newcommand{\nEpochs}{n}

\newcommand{\uGji}[1][\ind0,\ind1]{\smash{\widetilde{U}_{#1}}}

\newcommand{\Jr}[1][r]{J^{#1}}
\newcommand{\JrOf}[2][\rj]{\Jr[#1]\brk*{#2}}

\newcommand{\dx}{d_x}
\newcommand{\du}{d_u}

\newcommand{\Jdomain}{\mathcal{K}}

\newcommand{\steadyCov}[1][K]{\Sigma_{#1}}

\newcommand{\costToGoOf}[1][K]{P_{#1}}

\newcommand{\smoothCostOf}[2][r]{\cost^{#1}\brk*{#2}}

\newcommand{\costBound}{\nu}
\newcommand{\ind}[1]{\ifcase#1 j \or i \else s\fi}
\newcommand{\ui}[1][\ind1]{U_{#1}}
\newcommand{\pExp}{\rho}

\newcommand{\hist}[1][t]{\mathcal{H}_{#1}}
\newcommand{\filt}[1][t]{\mathcal{F}_{#1}}

\newcommand{\cropNoise}[1][t]{\tilde{w}_{#1}}
\newcommand{\cropNoiseBound}{\tilde{W}}
\newcommand{\cropNoiseMinStd}{\tilde{\sigma}}
\newcommand{\cropNoiseSet}{S}
\newcommand{\cropCost}[1][t]{\tilde{c}_{#1}}
\newcommand{\cropJof}[2][*]{\tilde{J}\brk#1{#2}}
\newcommand{\cropJstar}{\tilde{J}_\star}

\newcommand{\noiseCov}[1][w]{\Sigma_{#1}} 

\usepackage{enumitem}
\usepackage[numbers]{natbib}
\usepackage{crossreftools}
\usepackage{algorithm}
\usepackage[noend]{algpseudocode}

\algnewcommand{\IfThenElse}[3]{%
  \State \algorithmicif\ #1\ \algorithmicthen\ #2\ \algorithmicelse\ #3}

\begin{document}

\maketitle

\begin{abstract}
We consider the task of learning to control a linear dynamical system under fixed quadratic costs, known as the Linear Quadratic Regulator (LQR) problem. While model-free approaches are often favorable in practice, thus far only model-based methods, which rely on costly system identification, have been shown to achieve regret that scales with the optimal dependence on the time horizon~$T$. We present the first model-free algorithm that achieves similar regret guarantees. Our method relies on an efficient policy gradient scheme, and a novel and tighter analysis of the cost of exploration in policy space in this setting.
\end{abstract}

\section{Introduction}
\label{sec:intro}

Model-free, policy gradient algorithms have become a staple of Reinforcement Learning (RL) with both practical successes~\citep{lillicrap2015continuous,haarnoja2018soft}, and strong theoretical guarantees in several settings~\citep{sutton1999policy,silver2014deterministic}.
In this work we study the design and analysis of such algorithms for the adaptive control of Linear Quadratic Regulator (LQR) systems, as seen through the lens of regret minimization~\citep{abbasi2011regret,cohen2019learning,mania2019certainty}. In this continuous state and action reinforcement learning setting, an agent chooses control actions $u_t$ and the system state $x_t$ evolves according to the noisy linear dynamics
\begin{align*}
    x_{t+1}
    =
    \Astar x_t
    +
    \Bstar u_t
    +
    w_t
    ,
\end{align*}
where $\Astar$ and $\Bstar$ are transition matrices and $w_t$ are i.i.d zero-mean noise terms. The cost is a quadratic function of the current state and action, and the regret is measured with respect to the class of linear policies, which are known to be optimal for this setting.

Model-based methods, which perform planning based on a system identification procedure that estimates the transition matrices, have been studied extensively in recent years.
This started with \citet{abbasi2011regret}, which established an $\Oof[0]{\sqrt{T}}$ regret guarantee albeit with a computationally intractable method. More recently, \citet{cohen2019learning,mania2019certainty} complemented this result with computationally efficient methods, and \citet{cassel2020logarithmic,simchowitz2020naive} provided lower bounds, showing that this rate is generally unavoidable; regardless of whether the algorithm is model free or not.
In comparison, the best existing model-free algorithms are policy iteration procedures by \citet{krauth2019finite} and \citet{abbasi2019model} that respectively achieve $\OtildeOf[0]{T^{2/3}}$ and $\OtildeOf[0]{T^{2/3+\epsilon}}$ regret for $\epsilon = \Theta(1/\log{T})$.

Our main result is an efficient (in fact, linear time per step) policy gradient algorithm that achieves $\OtildeOf[0]{\sqrt{T}}$ regret, thus closing the (theoretical) gap between model based and free methods for the LQR model.
An interesting feature of our approach is that while the policies output by the algorithm are clearly state dependent, the tuning of their parameters requires no such access. Instead, we only rely on observations of the incurred cost, similar to bandit models (e.g., \citealp{cassel2020bandit}).

One of the main challenges of regret minimization in LQRs (and more generally, in reinforcement learning) is that it is generally infeasible to change policies as often as one likes. 
Roughly, this is due to a burn-in period following a policy change, during which the system converges to a new steady distribution, and typically incurs an additional cost proportional to the change in steady states, which is in turn proportional to the distance between policies.
There are several ways to overcome this impediment.
The simplest is to restrict the number of policy updates and explore directly in the action space via artificial noise (see e.g., \citealp{simchowitz2020naive}).
Another approach by \citet{cohen2019learning} considers a notion of slowly changing policies, however, these can be very prohibitive for exploration in policy space.
Other works (e.g., \citealp{agarwal2019logarithmic}) consider a policy parameterization that converts the problem into online optimization with memory, which also relies on slowly changing policies. This last method is also inherently model-based and thus not adequate for our purpose.

A key technical contribution that we make is to overcome this challenge by exploring directly in policy space. While the idea itself is not new, we provide a novel and tighter analysis that allows us to use larger perturbations, thus reducing the variance of the resulting gradient estimates. We achieve this by showing that the additional cost depends only quadratically on the exploration radius, which is a crucial ingredient for overcoming the $\Oof[0]{T^{2/3}}$ limitation.
The final ingredient of the analysis involves a sensitivity analysis of the gradient descent procedure that uses the estimated gradients. Here again, while similar analyses of gradient methods exist,
we provide a general result that gives appropriate conditions for which the optimization error depends only quadratically on the error in the gradients.

\paragraph{Related work.}

Policy gradient methods in the context of LQR has seen significant interest in recent years. Notably, \citet{fazel2018global} establish its global convergence in the perfect information setting, and give complexity bounds for sample based methods. 
Subsequently, \citet{malik2019derivative} improve the sample efficiency but their result holds only with a fixed probability and thus does not seem applicable for our purposes. 
\citet{hambly2020policy} also improve the sample efficiency, but in a finite horizon setting. 
\citet{mohammadi2020learning} give sample complexity bounds for the continuous-time variant of LQR.
Finally, \citet{tu2019gap} show that a model based method can potentially outperform the sample complexity of policy gradient by factors of the input and output dimensions. While we observe similar performance gaps in our regret bounds, these were not our main focus and may potentially be improved by a more refined analysis.
Moving away from policy gradients, \citet{yang2019global,jin2020analysis,yaghmaie2019using} analyze the convergence and sample complexity of other model free methods such as policy iteration and temporal difference (TD) learning, but they do not include any regret guarantees.

\section{Preliminaries}

\subsection{Setup: Learning in LQR}
We consider the problem of regret minimization in the LQR model.
At each time step $t$, a state $x_t \in \RR[\dx]$ is observed and action $u_t \in \RR[\du]$ is chosen.
The system evolves according to
\begin{align*}
	x_{t+1}
	=
	\Astar x_t + \Bstar u_t + w_t
	,
	\quad
	(x_0 = 0 ~ \text{w.l.o.g.})
	,
\end{align*}
where  the state-state $\Astar \in \RR[\dx \times \dx]$ and state-action $\Bstar \in \RR[\dx \times \du]$ matrices form the transition model and the $w_t$ are bounded, zero mean, i.i.d.~noise terms with a positive definite covariance matrix $\noiseCov \succ 0$. Formally, there exist $\noiseMinStd, \noiseBound > 0$ such that
\begin{align*}
    \EE w_t = 0
    \quad
    ,
    \norm{w_t} \le \noiseBound
    \quad
    ,
    \noiseCov
    =
    \EE w_t w_t\tran \succ \noiseMinStd^2 I
    .
\end{align*}
The bounded noise assumption is made for simplicity of the analysis, and in \cref{sec:gaussian} we show how to accommodate Gaussian noise via a simple reduction to this setting.
At time $t$, the instantaneous cost is 
$$
    c_t = x_t\tran Q x_t + u_t\tran R u_t,
$$ 
where $0 \prec Q,R \preceq I$ are positive definite. We note that the upper bound is without loss of generality since multiplying $Q$ and $R$ by a constant factor only re-scales the regret.

A policy of the learner is a potentially time dependent mapping from past history to an action $u \in \RR[\du]$ to be taken at the current time step.
Classic results in linear control establish that, given the system parameters $\Astar,\Bstar,Q$ and $R$, a linear transformation of the current state is an optimal policy for the infinite horizon setting.
We thus consider policies of the form $u_t = K x_t$ and define their infinite horizon expected cost,
\begin{align*}
	\Jof{K} = \lim_{T \to \infty} \frac{1}{T} \EEBrk{\sum_{t=1}^{T} x_t\tran \brk*{Q + K\tran R K} x_t},
\end{align*}
where the expectation is taken with respect to the random noise variables $w_t$. Let 
$
    \Kstar 
    =
    \argmin_K\Jof{K}
$
be a (unique) optimal policy and 
$
    \Jstar
    =
    \Jof{\Kstar}
$
denote the optimal infinite horizon expected cost, which are both well defined under mild assumptions.%
\footnote{These are valid under standard, very mild stabilizability assumptions (see \citealp{bertsekas1995dynamic}) that hold in our setting.}
We are interested in minimizing the \emph{regret} over~$T$ decision rounds, defined as
\begin{align*}
    \regret
    =
    \sum_{t=1}^{T} \brk!{x_t\tran Q x_t + u_t\tran R u_t - \Jstar}
    .
\end{align*}
We focus on the setting where the learner does not have a full a-priori description of the transition parameters $\Astar$ and $\Bstar$, and has to learn them while controlling the system and minimizing the regret.

Throughout, we assume that the learner has knowledge of constants $\costMatLower > 0$ and $\systemBound  \ge 1$ such that 
\begin{align*}
    \norm{Q^{-1}}, \norm{R^{-1}} \leq \ifrac{1}{\costMatLower}
    ,
    \text{ and }
    \norm{\Bstar} \le \systemBound
    .
\end{align*}
We also assume that there is a known stable (not necessarily optimal) policy $K_0$ and $\costBound >0$ such that $\Jof{K_0} \le \frac14\costBound$. We note that all of the aforementioned parameters could be easily estimated at the cost of an additive constant regret term by means of a warm-up period. However, recovering the initial control $K_0$ gives a constant that depends exponentially on the problem parameters as shown by \citet{chen2020black,mania2019certainty,cohen2019learning}.

Finally, denote the set of all ``admissable'' controllers
\begin{align*}
    \Jdomain
    =
    \brk[c]{K \mid \Jof{K} \le \costBound}
    .
\end{align*}
By definition, $K_0 \in \Jdomain$. As discussed below, over the set $\Jdomain$ the LQR cost function $\J$ has certain regularity properties that we will use throughout.

\subsection{Smooth Optimization}

\citet{fazel2018global} show that while the objective $\Jof{\cdot}$ is non-convex, it has properties that make it amenable to standard gradient based optimization schemes. We summarize these here as they are used in our analysis.
\begin{definition}[PL-condition]
\label{def:PL}
    A function $f : \fDomain \to \RR$ with global minimum $f^*$ is said to be $\pPL$-PL if it satisfies the Polyak-Lojasiewicz (PL) inequality with constant $\pPL > 0$, given by
    \begin{align*}
        \pPL(f(x) - f^*)
        \le
        \norm{\nabla f(x)}^2
        \quad,
        \forall x \in \fDomain
        .
    \end{align*}
\end{definition}

\begin{definition}[Smoothness]
\label{def:smooth}
	A function $f :\RR[d] \to \RR$ is locally $\pSmooth, \pRadius$-\emph{smooth} over $\fDomain \subseteq \RR[d]$ if for any $x \in \fDomain$ and $y \in \RR[d]$ with $\norm{y-x} \le \pRadius$
	\begin{align*}
	\norm{
	    \nabla f(x) - \nabla f(y)
    }
    \le 
    {\pSmooth}\norm{x - y}
    .
	\end{align*}
\end{definition}

\begin{definition}[Lipschitz]
\label{def:lipschitz}
    A function $f :\RR[d] \to \RR$ is locally $\pSmooth, \pRadius$-\emph{Lipschitz} over $\fDomain \subseteq \RR[d]$ if for any $x \in \fDomain$ and $y \in \RR[d]$ with $\norm{y-x} \le \pRadius$
	\begin{align*}
	\abs{
	    f(x) - f(y)
    }
    \le 
    {\pLip}\norm{x - y}
    .
	\end{align*}
\end{definition}

It is well-known that for functions satisfying the above conditions and for sufficiently small step size $\eta$, the gradient descent update rule
\begin{align*}
    x_{t+1} = x_t - \eta \nabla f(x_t)
\end{align*}
converges exponentially fast, i.e., there exists $0 \le \rho < 1$ such that
$
    {f(x_t) - f^*}
    \le
    \rho^t (f(x_0) - f^*)
$
(e.g., \citealp{nesterov2003introductory}).
This setting has also been investigated in the absence of a perfect gradient oracle. Here we provide a clean result that shows that the error in the optimization objective is only limited by the \emph{squared error} of any gradient estimate.

Finally, we require the notion of a one point gradient estimate~\citep{flaxman2005online}. Let $f : \fDomain \to \RR$ and define its smoothed version with parameter $r > 0$ as 
\begin{align}
\label{eq:smoothedVer}
	\smoothCostOf{x}
	=
	\EE[B]{\costOf{x + r B}}
	,
\end{align}
where $B \in \ball{d}$ is a uniform random vector over the Euclidean unit ball.
The following lemma is standard (we include a proof in \cref{sec:TechnicalLemmas} for completeness).

\begin{lemma}
\label{lemma:smoothFuncProperties}
    If $\cost$ is $(\pRadius, \pSmooth)$-\emph{locally smooth} and $r \le \pRadius$, then:
    \begin{enumerate}
        \item
        $
    		\nabla \smoothCostOf{x}
    		= 
    		\frac{d}{r} \EE[{\ui[]}]\brk[s]{\costOf{x + r \ui[]} \ui[]}
    		,
    	$
    	where $\ui[] \in \sphere{d}$ is a uniform random vector of the unit sphere;
        \item
        $
            \norm{\nabla \smoothCostOf{x} - \nabla \costOf{x}}
            \le
            \pSmooth r
            ,\;
            \forall x \in \fDomain
            .
        $
    \end{enumerate}
\end{lemma}

\subsection{Background on LQR}

It is well-known for the LQR problem that
\begin{align*}
    \Jof{K}
    =
    \tr{\costToGoOf[K] \noiseCov}
    =
    \tr{(Q + K\tran R K) \steadyCov[K]}
    ,
\end{align*}
where $\costToGoOf[K], \steadyCov[K]$ are the positive definite solutions to
\begin{align}
\label{eq:Pbellman}
    &\costToGoOf[K]
    =
    Q + K\tran R K + (\Astar+\Bstar K)\tran \costToGoOf[K] (\Astar+\Bstar K)
    ,
    \\
    &\steadyCov[K]
    =
    \noiseCov + (\Astar+\Bstar K) \steadyCov[K] (\Astar+\Bstar K)\tran
    .
\end{align}

Another important notion is that of strong stability \citep{cohen2018online}.
This is essentially a quantitative version of classic stability notions in linear control.
\begin{definition}[strong stability]
		A matrix $M$ is $(\kappa, \gamma)$-strongly stable (for $\kappa \ge 1$ and $0 < \gamma \le 1$) if there exists matrices $H \succ 0$ and $L$ such that $M = H L H^{-1}$ with $\norm{L} \le 1 - \gamma$ and $\norm{H}\norm{H^{-1}} \le \kappa$. A controller $K$ for is $(\kappa, \gamma)-$strongly stable if $\norm{K} \le \kappa$ and the matrix $\Astar +\Bstar K$ is $(\kappa, \gamma)$-strongly stable.
\end{definition}

The following lemma, due to \citet{cohen2019learning}, relates the infinite horizon cost of a controller to its strong stability parameters. 

\begin{lemma}[\citealp{cohen2019learning}, Lemma 18]
\label{lemma:costToStability}
	Suppose that $K \in \Jdomain$ then $K$ is $(\kappa, \gamma)-$strongly stable with 
	$
    	\kappa
    	=
    	\sqrt{\ifrac{\costBound}{\costMatLower\noiseMinStd^2}}
	$ 
	and 
	$
	    \gamma
	    =
	    \ifrac{1}{2\kappa^2}
	    .
	    $
\end{lemma}

The following two lemmas, due to \citet{cohen2018online,cassel2020logarithmic}, show that the state covariance converges exponentially fast, and that the state is bounded as long as controllers are allowed to mix.

\begin{lemma}[\citealp{cohen2018online}, Lemma 3.2]
\label{lemma:steadyState}
    Suppose we play some fixed $K \in \Jdomain$ starting from some $x_0 \in \RR[\dx]$, then
    \begin{align*}
        \norm{\EE\brk[s]{x_t x_t\tran} - \steadyCov[K]}
        &\le
        \kappa^2 e^{-2\gamma t} \norm{x_0 x_0\tran - \steadyCov[K]}
        ,
        \\
        \abs{\EE\brk[s]{\ct} - \Jof{K}}
        &\le
        \frac{\costBound \kappa^2}{\noiseMinStd^2} e^{-2\gamma t} \norm{x_0 x_0\tran - \steadyCov[K]}
        .
    \end{align*}
\end{lemma}
\begin{lemma}[\citealp{cassel2020logarithmic}, Lemma 39]
\label{lemma:stateBound}
	Let 
	$K_1, K_2, \ldots \in \Jdomain$. If we play each controller $K_i$ for at least $\tau \ge 2 \kappa^2 \log 2\kappa$ rounds before switching to $K_{i+1}$ then for all $t \ge 1$ we have that
	$
    	\norm{x_t}
    	\le
    	6\kappa^4 \noiseBound
	$
	and
	$
	    \ct
	    \le
	    36 \costBound \kappa^{8} \noiseBound^2 / \noiseMinStd^2
	    .
	$
\end{lemma}

The following is a summary of results from \citet{fazel2018global} that describe the main properties of $\steadyCov[K], \costToGoOf[K], \Jof{K}$. See \cref{sec:TechnicalLemmas} for the complete details.
\begin{lemma}[\citealp{fazel2018global}, Lemmas 11, 13, 16, 27 and 28]
\label{lemma:fazel}
    Let $K \in \Jdomain$ and $K' \in \RR[\du \times \dx]$ with
    \begin{align*}
        \norm{K - K'}
        \le
        \frac{1}{8 \systemBound \kappa^{3}}
        =
        \pRadius
        ,
    \end{align*}
    then we have that
    \begin{enumerate}
        \item
        $
        \tr{\costToGoOf[K]}
        \le
        \Jof{K} / \noiseMinStd^2
        ;
        $
        $
        \tr{\steadyCov[K]}
        \le
        \Jof{K} / \costMatLower
        ;
        $
        \item 
        $
        \norm{\steadyCov[K] - \steadyCov[K']}
        \le
        (\ifrac{8 \systemBound \costBound \kappa^3 }{\costMatLower})
        \norm{K - K'}
        ;
        $
        \item
        $
        \norm{\costToGoOf[K] - \costToGoOf[K']}
        \le
        16 \systemBound \kappa^7
        \norm{K - K'}
        ;
        $
        \item $\J$ satisfies the local Lipschitz condition (\cref{def:lipschitz}) over $\Jdomain$ with $\pRadius$ and
        $
        \pLip
        =
        \ifrac{4 \systemBound \costBound \kappa^7}{\costMatLower}
        ;
        $
        \item $\J$ satisfies the local smoothness condition (\cref{def:smooth}) over $\Jdomain$ with $\pRadius$ and
        $
        \pSmooth
        =
        \ifrac{112 \sqrt{\dx} \costBound \systemBound^2 \kappa^8}{\costMatLower}
        ;
        $
        \item $\J$ satisfies the PL condition (\cref{def:PL}) with 
        $
        \pPL
        =
        \ifrac{4 \costBound}{\kappa^4}
        .
        $
    \end{enumerate}
\end{lemma}

\section{Algorithm and Overview of Analysis}

We are now ready to present our main algorithm for model free regret minimization in LQR. The algorithm, given in \cref{alg:main}, optimizes an underlying controller $\Kj$ over epochs of exponentially increasing duration. 
Each epoch consists of sub-epochs, during which a perturbed controller $\Kji$ centered at $\Kj$ is drawn and played for $\tau$ rounds. 
At the end of each epoch, the algorithm uses $\cjitau$, which is the cost incurred during the final round of playing the controller $\Kji$,
to construct a gradient estimate which in turn is used to calculate the next underlying controller $\Kj[\ind0+1]$. 
Interestingly, we do not make any explicit use of the state observation $x_t$ which is only used implicitly to calculate the control signal, via $u_t = K_t x_t$.
Furthermore, the algorithm makes only $\Oof{\du \dx}$ computations per time step.

\begin{algorithm}[ht]
	\caption{LQR Online Policy Gradient} \label{alg:main}
	\begin{algorithmic}[1]
		\State \textbf{input:}
		    initial controller $K_0 \in \Jdomain$,
		    step size $\eta$,
		    mixing length $\tau$,
		    parameters $\pPL, \rj[0], \mj[0]$
	    \For{epoch $\ind0 = 0,1,2,\ldots$}
	        \State \textbf{set}
	        $
	            \rj
	            =
	            \rj[0] (1 - \pPL\eta/3)^{\ind0 / 2}
	            ,
	            \mj
	            =
	            \mj[0] (1 - \pPL\eta/3)^{-2 \ind0}
	        $
	        \For{$\ind1 = 1, \ldots, m_{\ind0}$}
	            \State \textbf{draw} $\uGji \in \RR[\du \times \dx]$ with i.i.d.\ $\gaussDist{0}{1}$ entries
	            \State \textbf{set}
	            $\uji = \uGji / \norm{\uGji}_F$
	            \State \textbf{play} $\Kji = \Kj + \rj \uji$ for $\tau$ rounds
	            \State \textbf{observe} the cost of the final round $\cjitau$
	        \EndFor
	        \State \textbf{calculate}
	        $
	            \gj
	            =
	            \frac{\dx \du}{\mj \rj} \sum_{\ind1=1}^{\mj} \cjitau \uji
            $
	        \State \textbf{update}
	        $
	            K_{\ind0 + 1}
	            =
	            \Kj - \eta \gj$
	    \EndFor
	\end{algorithmic}
\end{algorithm}

Our main result regarding \cref{alg:main} is stated in the following theorem: a high-probability $O(\sqrt{T})$ regret guarantee with a polynomial dependence on the problem parameters.

\begin{theorem}
\label{thm:main}
    Let 
    $
        \kappa
        =
        \sqrt{{\costBound}/{\costMatLower \noiseMinStd^2}}
    $
    and suppose we run \cref{alg:main} with parameters
    \begin{align*}
        &\eta
    	=
    	\frac{\costMatLower}{128 \costBound \systemBound^2 \kappa^{10}}
        ,\;
        \tau
        =
        2\kappa^2
        \log (7\kappa T)
        ,
        \\
        &\pPL
        =
        \frac{4\costBound}{\kappa^{4}}
        ,\;
        \rj[0]
        =
        \frac{\costMatLower}{448 \sqrt{\dx} \systemBound^2 \kappa^{10}}
        ,
        \\
        &\sqrt{\mj[0]}
        =
        \frac{2^{17} \du \dx^{3/2} \systemBound^2 \kappa^{20} \noiseBound^2}{\costMatLower \noiseMinStd^{2}} \sqrt{\log \frac{240 T^4}{\delta}}
        ,
    \end{align*}
    then with probability at least $1 - \delta$,
    \begin{align*}
        \regret
        =
        \Oof{
            \frac{\du \dx^{3/2} \systemBound^4 \kappa^{36} \noiseBound^2}{\costMatLower}
            \sqrt{T \tau \log \frac{T}{\delta}}
        }
        .
    \end{align*}
\end{theorem}

Here we give an overview of the main steps in proving \cref{thm:main}, deferring the details of each step to later sections.
Our first step is analyzing the utility of the policies $\Kj$ computed at the end of each epoch. We show that the regret of each $\Kj$ (over epoch $\ind0$) in terms of its long-term (steady state) cost compared to that of the optimal $\Kstar$, is controlled by the inverse square-root of the epoch length $\mj$.

\begin{lemma}[exploitation]
\label{lemma:exploitCost}
    Under the parameter choices of \cref{thm:main}, for any $\ind0 \ge 0$ we have that with probability at least $1 - \delta/8T^2$,
    \begin{align*}
        \Jof{\Kj}
        -
        \Jstar
        =
        \Oof{\costBound \sqrt{\frac{\mj[0]}{\mj}}}
        =
        \Oof{
        {\du \dx^{3/2} \systemBound^2 \kappa^{22} \noiseBound^2} \sqrt{\frac{1}{\mj} \log \frac{T}{\delta}}
        }
        ,
    \end{align*}
    and further that $\Jof{\Kj} \le \costBound / 2$.
\end{lemma}

The proof of the lemma is based on a careful analysis of gradient descent with inexact gradients and crucially exploits the PL and local-smoothness properties of the loss~$\Jof{\cdot}$. More details can be found in \cref{sec:optimization}.

The more interesting (and challenging) part of our analysis pertains to controlling the costs associated with exploration, namely, the penalties introduced by the perturbations of the controllers~$\Kj$.
The direct cost of exploration is clear: instead of playing the $\Kj$ intended for exploitation, the algorithm actually follows the perturbed controllers $K_{\ind0, \ind1}$ and thus incurs the differences in long-term costs $\Jof{K_{\ind0, \ind1}} - \Jof{\Kj}$.
Our following lemma bounds the accumulation of these penalties over an epoch $\ind0$; importantly, it shows that while the bound scales linearly with the length of the epoch $\mj$, it has a \emph{quadratic} dependence on the exploration radius $\rj$.

\begin{lemma}[direct exploration cost]
\label{lemma:KjiToKj}
    Under the parameter choices of \cref{thm:main}, for any $\ind0 \ge 0$ we have that with probability at least $1 - \delta/4T$,
    \begin{align*}
        \sum_{\ind1 = 1}^{\mj} \Jof{K_{\ind0, \ind1}} - \Jof{\Kj}
        =
        \Oof{
        \frac{\sqrt{\dx} \costBound \systemBound^2 \kappa^{8}}{\costMatLower} \rj^2 \mj
        +
        \costBound \sqrt{\mj \log \frac{T}{\delta}}
        }
        .
    \end{align*}
\end{lemma}

There are additional, indirect costs associated with exploration however: within each epoch the algorithm switches frequently between different policies, thereby suffering the indirect costs that stem from their ``burn-in'' period. This is precisely what gives rise to the differences between the realized cost $\cjis$ and the long-term cost $\Jof{\Kji}$ of the policy $\Kji$, the cumulative effect of which is bounded in the next lemma.
Here again, note the quadratic dependence on the exploration radius $\rj$ which is essential for obtaining our~$\sqrt{T}$-regret result.

\begin{lemma}[indirect exploration cost]
\label{lemma:lqrSwitch}
    Under the parameter choices of \cref{thm:main}, for any $\ind0 \ge 0$ we have that with probability at least $1 - \delta/4T$,
    \begin{align*}
        &\sum_{\ind1 = 1}^{\mj}
        \sum_{\ind2 = 1}^{\tau}
        \brk!{\cjis - \Jof{\Kji}}
        =
        \Oof{
        \frac{\costBound \kappa^8 \noiseBound^2}{\noiseMinStd^2}  \tau \sqrt{\mj \log \frac{T}{\delta}} 
        +
        \frac{\dx \costBound \systemBound^2  \kappa^{10}}{\costMatLower}
        \mj \rj^2
        }
        .
    \end{align*}
\end{lemma}

The technical details for \cref{lemma:KjiToKj,lemma:lqrSwitch} are discussed in \cref{sec:explorationCost}.
We now have all the main pieces required for proving our main result.

\begin{proof}[of \cref{thm:main}]
Taking a union bound, we conclude that \cref{lemma:lqrSwitch,lemma:KjiToKj,lemma:exploitCost} hold for all $\ind0 \ge 0$ with probability at least $1 - \delta$.
Now, notice that our choice of parameters is such that
\begin{align*}
    \rj^2 \mj
    =
    \rj[0]^2 \sqrt{\mj[0] \mj}
    =
    \Oof{
    \frac{\sqrt{\dx} \du \costMatLower \noiseBound^2}{\systemBound^2 \noiseMinStd^2}
    \sqrt{\mj \log \frac{T}{\delta}}
    }
    .
\end{align*}
Plugging this back into \cref{lemma:lqrSwitch,lemma:KjiToKj} we get that for all $\ind0$,
\begin{align*}
    \sum_{\ind1 = 1}^{\mj}
    \sum_{\ind2 = 1}^{\tau}
    \brk!{ \cjis \!-\! \Jof{\Kji} }
    &=
    \Oof{
    \frac{\du \dx^{3/2} \costBound \kappa^{10} \noiseBound^2}{\noiseMinStd^2}  \tau \sqrt{\mj \log \frac{T}{\delta}} 
    }
    ,
    \\
    \tau \sum_{\ind1 = 1}^{\mj} \Jof{K_{\ind0, \ind1}} \!-\! \Jof{\Kj}
    &=
    \Oof{
    \frac{\du \dx \costBound \kappa^{8} \noiseBound^2}{ \noiseMinStd^2}
    \tau \sqrt{\mj \log \frac{T}{\delta}}
    }
    .
\end{align*}
We conclude that the regret during epoch $\ind0$ is bounded as
\begin{align*}
    \sum_{\ind1 = 1}^{\mj}
    \sum_{\ind2 = 1}^{\tau}
    \brk!{ \cjis - \Jstar }
    &=
    \brk[s]*{\sum_{\ind1 = 1}^{\mj}
    \sum_{\ind2 = 1}^{\tau}
    \brk!{ \cjis - \Jof{\Kji} }
    }
    +
    \brk[s]*{\tau \sum_{\ind1 = 1}^{\mj} \Jof{K_{\ind0, \ind1}} - \Jof{\Kj}}
    +
    \brk[s]{\tau \mj (\Jof{\Kj} - \Jstar)}
    \\
    &=
    \Oof{
    {\du \dx^{3/2} \systemBound^2 \kappa^{22} \noiseBound^2} \tau \sqrt{\mj \log \frac{T}{\delta}}
    }
    ,
\end{align*}
where the second step also used the fact that $\costBound / \noiseMinStd^2 \le \kappa^2$.
Finally, a simple calculation (see \cref{lemma:mjSum}) shows that
\begin{align*}
    \sum_{\ind0=0}^{\nEpochs-1} \sqrt{\mj}
    =
    \Oof{\frac{1}{\pPL \eta}\sqrt{T / \tau}}
    =
    \Oof{
    \frac{\systemBound^2 \kappa^{14}}{\costMatLower} \sqrt{T / \tau}
    }
    ,
\end{align*}
and thus summing over the regret accumulated in each epoch concludes the proof.
\end{proof}

\section{Optimization Analysis}
\label{sec:optimization}

At its core, \cref{alg:main} is a policy gradient method with $\Kj$ being the prediction after $\ind0$ gradient steps. In this section we analyze the sub-optimality gap of the underlying controllers $\Kj$ culminating in the proof of \cref{lemma:exploitCost}. To achieve this, we first consider a general optimization problem with a corrupted gradient oracle, and show that the optimization rate is limited only by the square of the corruption magnitude. We follow this with an analysis of the LQR gradient estimation from which the overall optimization cost follows readily.

\subsection{Inexact First-Order Optimization}

Let $\cost : \RR[d] \to \RR$ be a function with global minimum $\costMin > -\infty$. Suppose there exists $\costLevel \in \RR$ such that $\cost$ is $\pPL$-\emph{PL}, $(\pRadius, \pSmooth)$-\emph{locally smooth}, and $(\pRadius,\pLip)$-\emph{locally Lipschitz} over the sub-level set
$
\fDomain
=
\brk[c]{x \mid \costOf{x} \le \costLevel}
.
$
We consider the update rule
\begin{align}
\label{eq:pgUpdateRule}
    x_{t+1}
    =
    x_t - \eta \gradOracle[t]
    ,
\end{align}
where $\costOf{x_0} \le \costLevel$, and $\gradOracle[t] \in \RR[d]$ is a corrupted gradient oracle that satisfies
\begin{align}
\label{eq:gradCorruption}
	\norm{\gradOracle[t] - \nabla \costOf{x_t}}
	\le
	\pCorrupt[t]
	,
\end{align}
where $\pCorrupt \le \min\brk[c]{\pLip,\sqrt{(\costLevel - \costMin)\pPL} / 2}$ is the magnitude of the corruption at step~$t$.
Define the effective corruption up to round $t$ as
\begin{align*}
    \pEffCorrupt^2
    =
    \max_{s \le t}\brk[c]*{
        \pCorrupt[s]^2
        \brk[s]{1 - (\pPL \eta / 3)}^{t-s}
    }
    ,
\end{align*}
and notice that if 
$
    \pCorrupt[s] \brk[s]{1 - (\pPL \eta / 3)}
    \le
    \pCorrupt[s+1]
$
then $\pEffCorrupt = \pCorrupt$.

The following result shows that this update rule achieves a linear convergence rate up to an accuracy that depends quadratically on the corruptions. 

\begin{theorem}[corrupted gradient descent]
\label{thm:corruptedConvergence}
	Suppose that 
	$
	\eta
	\le
	\min\brk[c]{
	    \ifrac{1}{\pSmooth}
	    ,
	    \ifrac{4}{\pPL}
	    ,
	    \ifrac{\pRadius}{2\pLip}
	    .
    }
    $
    Then for all $t \geq 0$,
	\begin{align*}
		\costOf{x_t} - \costMin
		\le
		\max\brk[c]*{
        \frac{4\pEffCorrupt[t-1]^2}{\pPL}
        ,
        \brk[s]*{1 - \frac{\pPL\eta}{3}}^t \brk*{\costOf{x_0} - \costMin}
        }
		,
	\end{align*}
	and consequently $x_t \in \fDomain$.
\end{theorem}

\begin{proof}
For ease of notation, denote $\gradCorrpution = \gradOracle - \nabla \costOf{x_t}$.
Now, suppose that $x_t \in \fDomain$, i.e., $\costOf{x_t} \le \costLevel$. Then we have that
\begin{align*}
    \norm{x_{t+1} - x_t}
    =
    \norm{\gradOracle[t]} \eta 
    \le
    (\norm{\nabla\costOf{x_t}} + \pCorrupt[t]) \pRadius / 2 \pLip
    \le
    \pRadius
    ,
\end{align*}
where the second step used our choice of $\eta$ and the third step used the Lipschitz assumption on $x_t$ and the bound on $\pCorrupt[t]$.
We conclude that $x_t, x_{t+1}$ satisfy the conditions for local smoothness and so we have that
\begin{align*}
	\costOf{x_{t+1}} - \costOf{x_t}
	&\le 
	\nabla\costOf{x_t}\tran \brk*{x_{t+1} - x_t} + \frac{\pSmooth}{2}\norm{x_{t+1} - x_t}^2
	\\
	&=
	\nabla\costOf{x_t}\tran \brk*{-\eta \gradOracle[t]}
	+
	\frac{\pSmooth}{2}
	\norm{\eta \gradOracle[t]}^2 
	\\
	&= 
	\eta \brk[s]*{
	    -
	    \norm{\nabla\costOf{x_t}}^2
	    -
	    \nabla\costOf{x_t}\tran \gradCorrpution
	    +
	    \frac{\eta\pSmooth}{2}
	    \norm{\nabla\costOf{x_t} + \gradCorrpution}^2
    } 
	\\
	&=
	\eta \brk[s]*{-\brk*{1 - \frac{\eta\pSmooth}{2}}\norm{\nabla\costOf{x_t}}^2 - \brk*{1 - \eta\pSmooth}\nabla\costOf{x_t}\tran \gradCorrpution + \frac{\eta\pSmooth}{2}\norm{\gradCorrpution}^2} 
	\\
	&=
	\eta \brk[s]*{- \brk*{1 - \eta\pSmooth}\norm{\frac{1}{2}\nabla\costOf{x_t} + \gradCorrpution}^2 -\brk*{\frac{3}{4} - \frac{\eta\pSmooth}{4}}\norm{\nabla\costOf{x_t}}^2 + \brk*{1 - \frac{\eta\pSmooth}{2}}\norm{\gradCorrpution}^2}
	\\
	&\le
	\eta \brk[s]*{-\brk*{\frac{3}{4} - \frac{\eta\pSmooth}{4}}\norm{\nabla\costOf{x_t}}^2 + \brk*{1 - \frac{\eta\pSmooth}{2}}\norm{\gradCorrpution}^2},
\end{align*}
where the last transition holds by choice of $\eta\pSmooth \le 1$. 
Next, using the PL condition and the bound on $\gradCorrpution$ (see \cref{eq:gradCorruption}) we get that
\begin{align*}
	\costOf{x_{t+1}} - \costOf{x_t}
	&\le 
	\eta \brk[s]*{
	    -\pPL\brk*{\frac{3}{4} - \frac{\eta\pSmooth}{4}}
	    \brk*{\costOf{x_t} - \costMin}
	    +
	    \brk*{1 - \frac{\eta\pSmooth}{2}}\pCorrupt^2}
	,
\end{align*}
and adding $\costOf{x_t} - \costMin$ to both sides of the equation we get
\begin{align*}
    \costOf{x_{t+1}} - \costMin
    \le
    \brk[s]*{1 - \frac{\pPL\eta}{4}\brk*{3 - \eta\pSmooth}}\brk*{\costOf{x_t} - \costMin} + \brk*{1 - \frac{\eta\pSmooth}{2}}\eta\pCorrupt^2.
\end{align*}
Now, if $\frac{4\pCorrupt^2}{\pPL} \le \costOf{x_t} - \costMin$ then since $\eta\pSmooth \le 1$ we have that
\begin{align}
\label{eq:thm1res1}
\begin{aligned}
    \costOf{x_{t+1}} - \costMin
    &\le
    \brk[s]*{
        1
        -
        \frac{\pPL\eta}{4}\brk*{3 - \eta\pSmooth} 
        +
        \frac{\pPL\eta}{4}\brk*{1 - \frac{\eta\pSmooth}{2}}
    }
    \brk*{\costOf{x_t} - \costMin}
    \\
    &= 
    \brk[s]*{1 - \frac{\pPL\eta}{2} \brk*{1 - \frac{\eta\pSmooth}{4}}}\brk*{\costOf{x_t} - \costMin}
    \\
    &\le
    \brk[s]*{1 - \frac{\pPL\eta}{3}} \brk*{\costOf{x_t} - \costMin}
    .
\end{aligned}
\end{align}
On the other hand, if $\frac{4\pCorrupt^2}{\pPL} \ge \costOf{x_t} - \costMin \ge 0$ then we have that
\begin{align}
\label{eq:thm1res2}
\begin{aligned}
    \costOf{x_{t+1}} - \costMin
    &\le 
    \brk[s]*{\max\brk[c]*{
        0
        ,
        \frac{4}{\pPL}
        -
        \eta\brk*{3 - \eta\pSmooth}
    }
    +
    \eta\brk*{1 - \frac{\eta\pSmooth}{2}}
    }\pCorrupt^2
    \\
    &\le
    \max\brk[c]*{
    \eta
    ,
    \frac{4}{\pPL} - \eta\brk*{2 - \eta\pSmooth}
    }\pCorrupt^2
    \le
    \frac{4\pCorrupt^2}{\pPL}
    ,
\end{aligned}
\end{align}
where the last transition holds, again, by our choice of $\eta$.
Combining \cref{eq:thm1res1,eq:thm1res2} we conclude that
\begin{align}
\label{eq:convergenceSingle}
    \costOf{x_{t+1}} - \costMin
    \le 
    \max\brk[c]*{
    \frac{4\pCorrupt^2}{\pPL}
    ,
    \brk[s]*{1 - \frac{\pPL\eta}{3}} \brk*{\costOf{x_t} - \costMin}
    }
    .
\end{align}
In particular, this implies that
$
\costOf{x_{t+1}}
\le
\max\brk[c]*{
    \costMin + \frac{4\pCorrupt^2}{\pPL}
    ,
    \costOf{x_t}
    }
\le
\costLevel
,
$
and thus $x_{t+1} \in \fDomain$. Since we assume that $x_0 \in \fDomain$, this completes an induction showing that $x_t \in \fDomain$ for all $t \ge 0$. We can thus unroll \cref{eq:convergenceSingle} recursively to obtain the final result.
\end{proof}

\subsection{Gradient Estimation}

The gradient estimate $\gj$ is a batched version of the typical one-point gradient estimator. We bound it in the next lemma using the following inductive idea: if $\Jof{\Kj} \le \costBound / 2$, then $\Kji \in \Jdomain$ and standard concentration arguments imply that the estimation error is small with high probability and thus \cref{thm:corruptedConvergence} implies that $\Jof{\Kj[\ind0+1]} \le \costBound / 2$.
\begin{lemma}[Gradient estimation error]
\label{lemma:GradEstBound}
    Under the parameter choices of \cref{thm:main}, for any $\ind0 \ge 0$ we have that with probability at least $1 - (\delta / 8T^3)$,
	\begin{align*}
		\norm{\gj - \nabla\Jof{\Kj}}_F
		\le
		\frac{\sqrt{\pPL \costBound}}{4}
		\brk*{1 - \frac{\pPL \eta}{3}}^{\ind0 / 2}
		.
	\end{align*}
\end{lemma}

\begin{proof}[of \cref{lemma:GradEstBound}]
Assume that conditioned on the event $\Jof{\Kj[\ind0]'} \le \costBound / 2$ for all $\ind0' \le \ind0$, the claim holds with probability at least $1 - \delta / 8 T^4$. 
We show by induction that we can peel-off the conditioning by summing the failure probability of each epoch. 
Concretely, we show by induction that the claim holds for all $\ind0' \le \ind0$ with probability at least $1 - \ind0 \delta / 8 T^4$. Since the number of epochs is less than $T$ (in fact logarithmic in $T$), this will conclude the proof.

The induction base follows immediately by our conditional assumption and the fact that $\Jof{\Kj[0]} \le \costBound / 4$. Now, assume the hypothesis holds up to $\ind0-1$. We show that the conditions of \cref{thm:corruptedConvergence} are satisfied with $\costLevel = \costBound / 2$ up to round $\ind0$, and thus $\Jof{\Kj[\ind0']} \le \costBound / 2$ for all $\ind0' \le \ind0$. We can then invoke our conditional assumption and a union bound to conclude the induction step.

We verify the conditions of \cref{thm:corruptedConvergence}. First, the Lipschitz, smoothness, and PL conditions hold by \cref{lemma:fazel}. Next,
notice that by definition $\Jstar \le \Jof{\Kj[0]} \le \costBound / 4$, and so by the induction hypothesis
$
    \norm{\gj[\ind0'] - \nabla\Jof{\Kj[\ind0']}}_F
    \le
    \sqrt{\costBound \pPL} / 4
    \le
    \sqrt{(\costLevel - \costMin)\pPL} / 2
    \le 
    \pLip
    ,
$
for all $\ind0' < \ind0$. Finally, noticing that $\kappa^2 > \dx$ it is easy to verify the condition on $\eta$.

It remains to show the conditional claim holds.
The event $\Jof{\Kj[\ind0']} \le \costBound / 2$ for all $\ind0' \le \ind0$ essentially implies that the policy gradient scheme did not diverge up to the start of epoch $\ind0$.
Importantly, this event is independent of any randomization during epoch $\ind0$ and thus will not break any i.i.d.~assumptions within the epoch.
Moreover, by \cref{lemma:fazel} and since $\rj[0] \le \costBound / 2 \pLip$, this implies that
$
    \Jof{\Kji[\ind0', \ind1]} \le \Jof{\Kj} + \pLip \rj \le \costBound
    ,
$
i.e., $\Kji[\ind0', \ind1] \in \Jdomain$ for all $\ind1$ and $\ind0' \le \ind0$.
For the remainder of the proof, we implicitly assume that this holds, allowing us to invoke \cref{lemma:steadyState,lemma:stateBound,lemma:fazel}. For ease of notation, we will not specify this explicitly. 

Now, let $\Jr$ be the smoothed version of $J$ as in \cref{eq:smoothedVer}.
Since $\rj \le \pRadius$ we can use \cref{lemma:smoothFuncProperties} to get that
\begin{align*}
    \norm{\gj - \nabla \Jof{\Kj}}_F
    &\le
    \norm{\gj - \nabla\JrOf{\Kj}}_F
    +
    \norm{\nabla\JrOf{\Kj} - \nabla\Jof{\Kj}}_F
    \\
    &\le
    \pSmooth \rj
    +
    \norm{\gj - \nabla\JrOf{\Kj}}_F
    ,
\end{align*}
Next, we decompose the remaining term using the triangle inequality to get that
\begin{align*}
	\norm{\gj - \nabla\JrOf{\Kj}}_F
	&=
	\norm*{
	    \frac{1}{\mj} \sum_{\ind1=1}^{\mj} \brk{
	    \frac{\dx \du}{\rj} \cjitau \uji
	    -
	    \nabla\JrOf{\Kj}
	    }
	}_F
	\\
	&\le
	\norm*{
	    \frac{1}{\mj} \sum_{\ind1=1}^{\mj} \brk{
	    \frac{\dx \du}{\rj} \Jof{\Kji} \uji
	    -
	    \nabla\JrOf{\Kj}
	    }
	}_F
	+
	\norm*{
	    \frac{1}{\mj} \sum_{\ind1=1}^{\mj} \brk{
	    \frac{\dx \du}{\rj} (\cjitau - \Jof{\Kji}) \uji
	    }
	}_F
	.
\end{align*}
By \cref{lemma:smoothFuncProperties}, we notice that, conditioned on $\Kj$, the first term is a sum of zero-mean i.i.d random vectors with norm bounded by $2 \du \dx \costBound / \rj$.
We thus invoke \cref{lemma:vectorAzuma} (Vector Azuma) to get that with probability at least $1 - \ifrac{\delta}{16T^4}$
\begin{align*}
    \norm*{
	    \frac{1}{\mj} \sum_{\ind1=1}^{\mj} 
	    \frac{\dx \du}{\rj} \Jof{\Kji} \uji
	    -
	    \nabla\JrOf{\Kj}
	}_F
    \le
    \frac{\du \dx \costBound}{\rj}
    \sqrt{\frac{8}{\mj} \log\frac{240 T^4}{\delta}}
    .
\end{align*}
Next, denote
$
    Z_{\ind1}
    =
    \frac{\dx \du}{\rj} (\cjitau - \Jof{\Kji}) \uji
    ,
$
and notice that the remaining term is exactly
$
    \norm{\frac{1}{\mj} \sum_{\ind1 = 1}^{\mj} Z_{\ind1}}_F
    .
$
Let $x_{\ind0, \ind1, \tau}$ be the state during the final round of playing controller $\Kji$, and $\filt[\ind1]$ be the filtration adapted to 
$
x_{\ind0, 1, \tau}
,
\ldots
,
x_{\ind0, \ind1, \tau}
,
\uji[\ind0, 1]
,
\ldots
,
\uji
.
$
We use \cref{lemma:steadyState} to get that
\begin{align*}
    \norm{\EE\brk[s]*{
    Z_{\ind1} \mid \filt[\ind1 - 1]
    }}_F
    &\le
    \EE\brk[s]*{\norm{\EE\brk[s]*{
    Z_{\ind1} \mid \filt[\ind1 - 1], \Kji
    }}_F \mid \filt[\ind1 - 1]}
    \\
    &\le
    \frac{\dx \du}{\rj}
    \EE\brk[s]*{
    \abs*{
        \EE\brk[s]*{\cjitau \mid \filt[\ind1 - 1], \Kji}
        -
        \Jof{\Kji}
    } \;\Big|\; \filt[\ind1 - 1]}
    \\
    &\le
    \frac{\dx \du \costBound \kappa^2}{\rj \noiseMinStd^2} e^{-2\gamma \tau}
    \EE\brk[s]*{
    \norm{x_{\ind0, \ind1 - 1, \tau} x_{\ind0, \ind1 - 1, \tau}\tran - \steadyCov[\Kji]} \;\Big|\; \filt[\ind1 - 1]}
    \\
    &\le
    \frac{37 \dx \du \costBound \kappa^{10} \noiseBound^2}{\rj \noiseMinStd^2} e^{-2\gamma \tau}
    \\
    &
    \le
    \frac{\dx \du \costBound \kappa^{8} \noiseBound^2}{\rj \noiseMinStd^2 T^2}
    ,
\end{align*}
where the last step plugged in the value of $\tau$ and the one before that used \cref{lemma:stateBound,lemma:fazel} to bound $\norm{\steadyCov[\Kji]} \le \costBound / \costMatLower = \kappa^2 \noiseMinStd^2$ and $\norm{x_{\ind0, \ind1 - 1, \tau}} \le 6 \kappa^4 \noiseBound$.
Further using \cref{lemma:stateBound} to bound $\cjitau$, we also get that
\begin{align*}
    \norm{
        Z_{\ind1}
        -
        \EE\brk[s]*{
            Z_{\ind1}
            \mid
            \filt[\ind1 - 1]
        }
    }_F
    \le
    \norm{Z_{\ind1}}_F
    +
    \norm{\EE\brk[s]*{Z_{\ind1} \mid \filt[\ind1 - 1]}}_F
    \le
    \frac{\dx \du \cjitau}{\rj}
    +
    \norm{\EE\brk[s]*{Z_{\ind1} \mid \filt[\ind1 - 1]}}_F
    \le
    \frac{37 \dx \du \costBound \kappa^{8} \noiseBound^2}{\rj \noiseMinStd^2}
    .
\end{align*}
Since $Z_{\ind1}$ is $\filt[\ind1]-$measurable we can invoke \cref{lemma:vectorAzuma} (Vector Azuma) to get that with probability at least $1 - \frac{\delta}{16T^4}$,
\begin{align*}
    \norm*{\frac{1}{\mj} \sum_{\ind1 = 1}^{\mj} Z_{\ind1}}_F
    &\le
    \frac{1}{\mj}\norm*{\sum_{\ind1 = 1}^{\mj}
        Z_{\ind1}
        -
        \EE\brk[s]*{
            Z_{\ind1}
            \mid
            \filt[\ind1 - 1]
        }
    }_F
    +
    \frac{1}{\mj}\sum_{\ind1 = 1}^{\mj} \norm{\EE\brk[s]*{
    Z_{\ind1} \mid \filt[\ind1 - 1]
    }}_F
    \\
    &\le
    \frac{\dx \du \costBound \kappa^{8} \noiseBound^2}{\rj \noiseMinStd^2}
    \brk[s]*{
        37 \sqrt{\frac{2}{\mj}\log\frac{240 T^4}{\delta}}
        +
        \frac{1}{T^2}
    }
    \\
    &\le
    \frac{54 \dx \du \costBound \kappa^{8} \noiseBound^2}{\rj \noiseMinStd^2}
        \sqrt{\frac{1}{\mj}\log\frac{240 T^4}{\delta}}
    .
\end{align*}
Using a union bound and putting everything together, we conclude that with probability at least $1 - (\delta / 8T^4)$,
\begin{align*}
    \norm{\gj - \nabla \Jof{\Kj}}_F
    &\le
    \pSmooth \rj 
    +
    \frac{54 \dx \du \costBound \kappa^{8} \noiseBound^2}{\rj \noiseMinStd^2}
    \sqrt{\frac{1}{\mj}\log\frac{240 T^4}{\delta}}
    \\
    &=
    \brk[s]*{
        \pSmooth \rj[0]
        +
        \frac{54 \dx \du \costBound \kappa^{8} \noiseBound^2}{\noiseMinStd^2 \rj[0] \mj[0]^{1/2}}
        \sqrt{\log\frac{240 T^4}{\delta}}
    }
    \brk*{1 - \frac{\pPL \eta}{3}}^{\ind0 / 2}
    \\
    &\le
    2\pSmooth \rj[0]
    \brk*{1 - \frac{\pPL \eta}{3}}^{\ind0 / 2}
    \\
    &\le
    \frac{\sqrt{\pPL \costBound}}{4}
    \brk*{1 - \frac{\pPL \eta}{3}}^{\ind0 / 2}
    ,
\end{align*}
where the last steps plugged in the values of $\pPL, \pSmooth, \rj[0]$, and $\mj[0]$.
\end{proof}

\subsection{Proof of \cref{lemma:exploitCost}}
\cref{lemma:exploitCost} is a straightforward consequence of the previous results.

\begin{proof}
    For $\ind0 = 0$ the claim holds trivially by our assumption that $\Jof{\Kj[0]} \le \costBound / 4$. Now, for $\ind0 \ge 1$,
    we use a union bound on \cref{lemma:GradEstBound} to get that with probability at least $1 - \delta / 8T^2$ 
    \begin{align*}
        \norm{\gj - \nabla\Jof{\Kj}}
		\le
		\frac{\sqrt{\pPL \costBound}}{4}
		\brk*{1 - \frac{\pPL \eta}{3}}^{\ind0 / 2}
        ,
		\qquad
        \forall \ind0 \ge 0
        .
    \end{align*}
    Then by \cref{thm:corruptedConvergence} we have that
    \begin{align*}
        \Jof{\Kj}
        \le
        \Jstar
        +
        \frac{\costBound}{4}
        \brk*{1 - \frac{\pPL \eta}{3}}^{\ind0 - 1}
        \le
        \min\brk[c]*{
            \frac{\costBound}{2}
            ,
            \Jstar
            +
            \frac{\costBound}{2}
            \brk*{1 - \frac{\pPL \eta}{3}}^{\ind0}
        }
        ,
    \end{align*}
    where the last step used the facts that
    $
        \Jstar
        \le
        \Jof{\Kj[0]}
        \le
        \costBound / 4
    $
    and
    $1 - \pPL \eta / 3 \ge 1/2$.
\end{proof}

\section{Exploration Cost Analysis}
\label{sec:explorationCost}

In this section we demonstrate that exploring near a given initial policy does not incur linear regret in the exploration radius (as more straightforward arguments would give), and use this crucial observation for proving \cref{lemma:KjiToKj,lemma:lqrSwitch}.

We begin with \cref{lemma:lqrSwitch}. The main difficulty in the proof is captured by the following basic result, 
which roughly shows that the expected cost for transitioning between two i.i.d.~copies of a given random policy scales with the \emph{variance} of the latter. This would in turn give the quadratic dependence on the exploration radius we need.
\begin{lemma}
\label{lemma:singleSwitch}
    Let $K \in \Jdomain$ be fixed. Suppose $K_1, K_2$ are i.i.d.~random variables such that $\EE K_i = K$, and 
    $
    \norm{K_i - K}_F 
    \le
    r
    \le
    \pRadius
    .
    $
    If $x_\tau(K_1)$ is the result of playing $K_1$ for $\tau \ge 1$ rounds starting at $x_0 \in \RR[\dx]$, then
    \begin{align*}
        \EE\brk[s]{
        x_\tau(K_1) \tran (\costToGoOf[K_2] - \costToGoOf[K_1]) x_\tau(K_1)
        }
        \le
        \frac{256 \dx \costBound \systemBound^2  \kappa^{10}}{\costMatLower}
        r^2
        +
        32 \dx \systemBound \kappa^9 (\norm{x_0}^2 + \kappa^2 \noiseMinStd^2) r e^{-2\gamma \tau}
        .
    \end{align*}
\end{lemma}
\begin{proof}
    Notice that the expectation is with respect to both controllers and the $\tau$ noise terms, all of which are jointly independent.
    We begin by using \cref{lemma:steadyState,lemma:fazel} to get that
    \begin{align*}
        {\tr{(\costToGoOf[K_2] - \costToGoOf[K_1]) (\EE\brk[s]{x_\tau(K_1) x_\tau(K_1)\tran \mid K_1} - \steadyCov[K_1])}
        }
        &\le
        32 \dx \systemBound \kappa^7 r  \norm{\EE\brk[s]{x_\tau(K_1) x_\tau(K_1)\tran \mid K_1} - \steadyCov[K_1])}
        \\
        &\le
        32 \dx \systemBound \kappa^9 r
        e^{-2\gamma \tau} \norm{x_0 x_0\tran - \steadyCov[K_1]}
        \\
        &\le
        32 \dx \systemBound \kappa^9 (\norm{x_0}^2 + \kappa^2 \noiseMinStd^2) r e^{-2\gamma \tau}
        ,
    \end{align*}
    where the last step also used the fact that $\kappa^2 \noiseMinStd^2 = \costBound / \costMatLower$.
    Now, since $\costToGoOf[K_1], \costToGoOf[K_2]$ do not depend on the noise, we can use the law of total expectation to get that
    \begin{align*}
        \EE\brk[s]{x_\tau(K_1)\tran (\costToGoOf[K_2] - \costToGoOf[K_1]) x_\tau(K_1)}
        &=
        \EE\brk[s]{ \tr{(\costToGoOf[K_2] - \costToGoOf[K_1]) \EE\brk[s]{x_\tau(K_1) x_\tau(K_1)\tran \mid K_1}}
        }
        \\
        &\le
        \EE\brk[s]{ \tr{(\costToGoOf[K_2] - \costToGoOf[K_1]) \steadyCov[K_1]}
        }
        +
        4 \dx \costMatLower \kappa^2 (\norm{x_0}^2 + \kappa^2 \noiseMinStd^2) e^{-2\gamma \tau}
        .
    \end{align*}
    To bound the remaining term, notice that since $K_1, K_2$ are i.i.d, we may change their roles without changing the expectation, i.e.,
    \begin{align*}
        \EE\brk[s]{ \tr{(\costToGoOf[K_2] - \costToGoOf[K_1]) \steadyCov[K_1]}
        }
        =
        \EE\brk[s]{ \tr{(\costToGoOf[K_1] - \costToGoOf[K_2]) \steadyCov[K_2]}
        }
        ,
    \end{align*}
    we conclude that
    \begin{align*}
        \EE\brk[s]{ \tr{(\costToGoOf[K_2] - \costToGoOf[K_1]) \steadyCov[K_1]}
        }
        &=
        \frac12 \EE\brk[s]{ \tr{(\costToGoOf[K_2] - \costToGoOf[K_1]) (\steadyCov[K_1] - \steadyCov[K_2])}
        }
        \\
        &\le
        \frac{\dx}{2}\norm{\costToGoOf[K_2] - \costToGoOf[K_1]} \norm{\steadyCov[K_2] - \steadyCov[K_1]}
        \\
        &\le
        \frac{256 \dx \costBound \systemBound^2  \kappa^{10}}{\costMatLower}
        r^2
        ,
    \end{align*}
    where the last step also used \cref{lemma:fazel}.
\end{proof}

\subsection{Proof of \cref{lemma:lqrSwitch}}
\label{sec:lqrSwitchProof}

Before proving \cref{lemma:lqrSwitch} we introduce a few simplifying notations. Since the lemma pertains to a single epoch, we omit its notation $\ind0$ wherever it is clear from context. For example, $\Kji$ will be shortened to $\Kji[\ind1]$ and $x_{\ind0,\ind1,\ind2}$ to $x_{\ind1,\ind2}$. In any case, we reserve the index $\ind0$ for epochs and $\ind1$ for sub-epochs. In this context, we also denote
the gap between realized and idealized costs during sub-epoch $\ind1$ by
\begin{align*}
    \dCi
    =
    \sum_{\ind2=1}^{\tau} (\cjis[\ind1,\ind2] - \Jof{\Kji[\ind1])}
    ,
\end{align*}
and the filtration $\hist[\ind1]$ adapted to $w_{1, 1}, \ldots, w_{\ind1, \tau-1}, \Kji[1], \ldots, \Kji[\ind1]$.
We note that $\Kji[\ind1]$ and $\dCi$ are $\hist[\ind1]-$measurable.
The following lemma uses \cref{eq:Pbellman} to decompose the cost gap at the various time resolutions.
See proof at the end of this section.

\begin{lemma}
\label{lemma:dCi}
    If the epoch initial controller satisfies $\Jof{\Kj} \le \costBound / 2$ then
	(recall that $\costToGoOf[K]$ is the positive definite solution to \cref{eq:Pbellman}):
	\begin{enumerate}
	    \item 
	    $
	        \cjis[\ind1, \ind2] - \Jof{\Kji[\ind1]}
	        =
	        x_{\ind1,\ind2}\tran \costToGoOf[{\Kji[\ind1]}] x_{\ind1,\ind2} - \EE[w_{\ind1,\ind2}]\brk[s]{x_{\ind1,\ind2+1}\tran \costToGoOf[{\Kji[\ind1]}] x_{\ind1,\ind2+1}}
	        ;
	    $
	    \item
	    $
	        \EE\brk[s]{\dCi \mid \hist[\ind1-1]}
	        =
	        \EE\brk[s]*{
	        x_{\ind1, 1}\tran \costToGoOf[{\Kji[\ind1]}] x_{\ind1, 1}
	        -
	        x_{\ind1+1,1}\tran \costToGoOf[{\Kji[\ind1]}] x_{\ind1+1,1}
	        \mid
	        \hist[\ind1-1]
	        }
	        ;
	    $
	    \item
	    $
	        \sum_{\ind1=1}^{\mj} \EE\brk[s]{\dCi \mid \hist[\ind1-1]}
	        \le
	        \EE\brk[s]{x_{1,1}\tran \costToGoOf[{\Kji[1]}] x_{1,1}}
            +
            \sum_{\ind1=2}^{\mj} \big(
                \EE\brk[s]{x_{\ind1, 1}\tran \costToGoOf[{\Kji[\ind1]}] x_{\ind1, 1} \mid \hist[\ind1-1]}
                -
                \EE\brk[s]{x_{\ind1, 1}\tran \costToGoOf[{\Kji[\ind1-1]}] x_{\ind1,1} \mid \hist[\ind1-2]}
                \big)
	        .
	    $
	\end{enumerate}
\end{lemma}

We are now ready to prove the main lemma of this section.
\begin{proof}[of \cref{lemma:lqrSwitch}]
    First, by \cref{lemma:exploitCost}, the event $\Jof{\Kj[\ind0']} \le \costBound / 2$ for all $\ind0' \le \ind0$ holds with probability at least $1 - \delta / 8T$.
    As in the proof of \cref{lemma:GradEstBound}, we will implicitly assume that this event holds, which will not break any i.i.d assumptions during epoch $\ind0$ and implies that $\Kji[\ind1] \in \Jdomain$ for all $1 \le \ind1 \le \mj$.
    We also use this to invoke \cref{lemma:stateBound,lemma:fazel} to get that for any $1 \le \ind1, \ind1' \le \mj$ and $1 \le \ind2 \le \tau$ we have
    $
        x_{\ind1,\ind2}\tran \costToGoOf[{\Kji[\ind1']}] x_{\ind1,\ind2}
        \le
        36 \costBound \kappa^8 \noiseBound^2 / \noiseMinStd^2
        =
        \nu_0
        .
    $
    
    Now, recall that $\dCi$ is $\hist[\ind1]$-measurable and thus $\dCi - \EE\brk[s]{\dCi \mid \hist[\ind1-1]}$ is a martingale difference sequence. Using the first part of \cref{lemma:dCi} we also conclude that each term bounded by $\tau \nu_0$. Applying Azuma's inequality we get that with probability at least $1 - (\delta / 16T)$
    \begin{align*}
        \sum_{\ind1=1}^{\mj} \dCi
        &=
        \sum_{\ind1=1}^{\mj} \dCi - \EE\brk[s]{\dCi \mid \hist[\ind1-1]} + \EE\brk[s]{\dCi \mid \hist[\ind1-1]}
        \\
        &\le
        \sqrt{2 \mj \tau^2 \nu_0^2 \log \frac{16 T}{\delta}}
        +
        \sum_{\ind1=1}^{\mj} \EE\brk[s]{\dCi \mid \hist[\ind1-1]}
        .
    \end{align*}
    Now, recall from \cref{lemma:dCi} that
    \begin{align*}
        \sum_{\ind1=1}^{\mj}& \EE\brk[s]{\dCi \mid \hist[\ind1-1]}
        \\
        &\le
        \EE\brk[s]{x_{1,1}\tran \costToGoOf[{\Kji[1]}] x_{1,1}}
        +
        \sum_{\ind1=2}^{\mj}
            \EE\brk[s]{x_{\ind1, 1}\tran \costToGoOf[{\Kji[\ind1]}] x_{\ind1, 1} \mid \hist[\ind1-1]}
            -
            \EE\brk[s]{x_{\ind1, 1}\tran \costToGoOf[{\Kji[\ind1-1]}] x_{\ind1,1} \mid \hist[\ind1-2]}
        \\
        &=
        \EE\brk[s]{x_{1,1}\tran \costToGoOf[{\Kji[1]}] x_{1,1}}
        +
        \sum_{\ind1=2}^{\mj}
            \EE\brk[s]{x_{\ind1, 1}\tran \costToGoOf[{\Kji[\ind1]}] x_{\ind1, 1} \mid \hist[\ind1-1]}
            -
            \EE\brk[s]{x_{\ind1, 1}\tran \costToGoOf[{\Kji[\ind1]}] x_{\ind1, 1} \mid \hist[\ind1-2]}
            +
            \EE\brk[s]{x_{\ind1, 1}\tran 
            (
            \costToGoOf[{\Kji[\ind1]}]
            -
            \costToGoOf[{\Kji[\ind1-1]}]
            ) 
            x_{\ind1, 1} \mid \hist[\ind1-2]}
        .
    \end{align*}
    The first two terms in the sum form a martingale difference sequence with each term being bound by $\nu_0$. We thus have that with probability at least $1 - \delta/16 T$,
    \begin{align*}
        \sum_{\ind1=1}^{\mj} \EE\brk[s]{\dCi \mid \hist[\ind1-1]}
        \le
        \nu_0
        +
        \sqrt{2 \mj \nu_0^2 \log \frac{16 T}{\delta}} 
        +
        \sum_{\ind1=2}^{\mj}
            \EE\brk[s]{x_{\ind1, 1}\tran 
            (
            \costToGoOf[{\Kji[\ind1]}]
            -
            \costToGoOf[{\Kji[\ind1-1]}]
            ) 
            x_{\ind1, 1} \mid \hist[\ind1-2]}
        .
    \end{align*}
    Notice that the summands in remaining term fit the setting of \cref{lemma:singleSwitch} and thus
    \begin{align*}
        \sum_{\ind1=2}^{\mj}
            \EE\brk[s]{x_{\ind1, 1}\tran 
            (
            \costToGoOf[{\Kji[\ind1]}]
            -
            \costToGoOf[{\Kji[\ind1-1]}]
            ) 
            x_{\ind1, 1} \mid \hist[\ind1-2]}
        &\le
        \frac{256 \dx \costBound \systemBound^2  \kappa^{10}}{\costMatLower}
        \rj^2 \mj
        +
        \sum_{\ind1=1}^{\mj}
        32 \dx \systemBound \kappa^9 (\norm{x_{\ind1, 1}}^2 + \kappa^2 \noiseMinStd^2) \rj e^{-2\gamma \tau}
        \\
        &
        \le
        \frac{256 \dx \costBound \systemBound^2  \kappa^{10}}{\costMatLower}
        \rj^2 \mj
        +
        \frac{25 \dx \systemBound \kappa^{15} \noiseBound^2 \rj \mj}{T^2}
        \\
        &
        \le
        \frac{257 \dx \costBound \systemBound^2  \kappa^{10}}{\costMatLower}
        \rj^2 \mj
        ,
    \end{align*}
    where the second transition plugged in $\tau$ and used \cref{lemma:stateBound} to bound $\norm{x_{\ind1, 1}}$, and the third transition used the fact that $T^{-2} \le \mj^{-2} \le \rj / \mj[0]$.
    Plugging in the value of $\nu_0$ and using a union bound, we conclude that with probability at least $1 - \delta / 4 T$,
    \begin{align*}
        \sum_{\ind1=1}^{\mj} \dCi
        \le
        \frac{144 \costBound \kappa^8 \noiseBound^2}{\noiseMinStd^2} \tau \sqrt{\mj \log \frac{16 T}{\delta}} 
        +
        \frac{257 \dx \costBound \systemBound^2  \kappa^{10}}{\costMatLower}
        \rj^2 \mj
        ,
    \end{align*}
    as desired.
\end{proof}

\begin{proof}[of \cref{lemma:dCi}]
    By our assumption that $\Jof{\Kj} \le \costBound / 2$ we have that $\Jof{\Kji[\ind1]} \le \costBound$ and thus $\costToGoOf[{\Kji[\ind1]}]$ is well defined.
	Now, recall that $x_{\ind1, \ind2+1} = \brk{\Astar + \Bstar \Kji[{\ind1}]} x_{\ind1, \ind2} + w_{\ind1,\ind2}$ and $\Jof{\Kji[\ind1]} = \EE[w_{\ind1,\ind2}]\brk[s]{w_{\ind1,\ind2}\tran \costToGoOf[{\Kji[\ind1]}] w_{\ind1,\ind2}}$ where $\costToGoOf[{\Kji[\ind1]}]$ satisfies \cref{eq:Pbellman} with $K=\Kji[\ind1]$. Then we have that
	\begin{align*}
		\EE[w_{\ind1,\ind2}]\brk[s]{x_{\ind1,\ind2+1}\tran \costToGoOf[{\Kji[\ind1]}] x_{\ind1,\ind2+1}}
		&=
		\EE[w_{\ind1,\ind2}]\brk[s]{((\Astar + \Bstar \Kji[\ind1]) x_{\ind1,\ind2} + w_{\ind1,\ind2})\tran \costToGoOf[{\Kji[\ind1]}] ((\Astar + \Bstar \Kji[\ind1]) x_{\ind1,\ind2} + w_{\ind1,\ind2})} \\
		&=
		((\Astar + \Bstar \Kji[\ind1]) x_{\ind1,\ind2})\tran \costToGoOf[K_{t}] ((\Astar + \Bstar \Kji[\ind1]) x_{\ind1,\ind2})
		+
		\EE[w_{\ind1,\ind2}]\brk{w_{\ind1,\ind2}\tran \costToGoOf[{\Kji[\ind1]}] w_{\ind1,\ind2}} \\
		&=
		((\Astar + \Bstar \Kji[\ind1]) x_{\ind1,\ind2})\tran \costToGoOf[K_{t}] ((\Astar + \Bstar \Kji[\ind1]) x_{\ind1,\ind2})
		+
		\Jof{\Kji[\ind1]}.
	\end{align*}
	Now, multiplying \cref{eq:Pbellman} by $x_{\ind1,\ind2}$ from both sides we get that
	\begin{align*}
		x_{\ind1,\ind2}\tran \costToGoOf[{\Kji[\ind1]}] x_{\ind1,\ind2}
		&=
		x_{\ind1,\ind2}\tran \brk{Q + \Kji[\ind1]\tran R K_t} x_{\ind1,\ind2}
		+
		((\Astar + \Bstar \Kji[\ind1]) x_{\ind1,\ind2})\tran \costToGoOf[{\Kji[\ind1]}] ((\Astar + \Bstar \Kji[\ind1]) x_{\ind1,\ind2})
		\\
		&=
		\cjis[\ind1,\ind2]
		+
		\EE[w_{\ind1,\ind2}]\brk[s]{x_{\ind1,\ind2+1}\tran \costToGoOf[{\Kji[\ind1]}] x_{\ind1,\ind2+1}}
		-
		\Jof{\Kji[\ind1]},
	\end{align*}
	where the second transition plugged in the previous equality. Changing sides concludes the first part of the proof.
	For the second part, notice that taking expectation with respect to $w_{\ind1,\ind2}$ is equivalent to conditional expectation with respect to all past epochs and $w_{1,1}, \ldots, w_{\ind1, \ind2-1}, \Kji[1], \ldots, \Kji[\ind1]$ of the current epoch. Since for all $1 \le \ind2 \le \tau$ this contains $\hist[\ind1-1]$, we use the law of total expectation to get that
	\begin{align*}
	    \EE\brk[s]{\cjis[\ind1,\ind2] - \Jof{\Kji[\ind1]} \mid \hist[\ind1-1]}
	    &=
	    \EE\brk[s]{x_{\ind1,\ind2}\tran \costToGoOf[{\Kji[\ind1]}] x_{\ind1,\ind2} \mid \hist[\ind1-1]}
		-
		\EE\brk[s]{\EE[w_{\ind1,\ind2}]\brk[s]{x_{\ind1,\ind2+1}\tran \costToGoOf[{\Kji[\ind1]}] x_{\ind1,\ind2+1}} \mid \hist[\ind1-1]}
		\\
		&=
	    \EE\brk[s]{x_{\ind1,\ind2}\tran \costToGoOf[{\Kji[\ind1]}] x_{\ind1,\ind2} \mid \hist[\ind1-1]}
		-
		\EE\brk[s]{x_{\ind1,\ind2+1}\tran \costToGoOf[{\Kji[\ind1]}] x_{\ind1,\ind2+1} \mid \hist[\ind1-1]}
		.
	\end{align*}
	Summing over $\ind2$, noticing that the sum is telescopic, and that time $(\ind1, \tau+1)$ is in fact the start of the next sub-epoch, i.e., $(\ind1+1, 1)$, concludes the second part of the proof.
	Finally, we sum over $\ind1$ to get that
	\begin{align*}
	    &\sum_{\ind1=1}^{\mj}  \EE\brk[s]{\dCi \mid \hist[\ind1-1]}
	    =
	    \sum_{\ind1=1}^{\mj} \EE\brk[s]*{x_{\ind1,1}\tran \costToGoOf[{\Kji[\ind1]}] x_{\ind1,1} - x_{\ind1+1,1}\tran \costToGoOf[{\Kji[\ind1]}] x_{\ind1+1,1}
        \mid
        \hist[\ind1-1]
        }
        \\
        &=
        \EE\brk[s]{x_{1,1}\tran \costToGoOf[{\Kji[1]}] x_{1,1}}
        -
        \EE\brk[s]{x_{\mj+1,1}\tran \costToGoOf[{\Kji[\mj]}] x_{\mj+1,1} \mid \hist[\mj-1]}
        +
        \sum_{\ind1=2}^{\mj}
            \EE\brk[s]{x_{\ind1,1}\tran \costToGoOf[{\Kji[\ind1]}] x_{\ind1,1} \mid \hist[\ind1-1]}
            -
            \EE\brk[s]{x_{\ind1,1}\tran \costToGoOf[{\Kji[\ind1-1]}] x_{\ind1,1} \mid \hist[\ind1-2]}
        \\
        &\le
        \EE\brk[s]{x_{1,1}\tran \costToGoOf[{\Kji[1]}] x_{1,1}}
        +
        \sum_{\ind1=2}^{\mj}
            \EE\brk[s]{x_{\ind1,1}\tran \costToGoOf[{\Kji[\ind1]}] x_{\ind1,1} \mid \hist[\ind1-1]}
            -
            \EE\brk[s]{x_{\ind1,1}\tran \costToGoOf[{\Kji[\ind1-1]}] x_{\ind1,1} \mid \hist[\ind1-2]}
        ,
	\end{align*}
	concluding the third part of the proof.
\end{proof}

\subsection{Proof of \cref{lemma:KjiToKj}}

\begin{proof}[of \cref{lemma:KjiToKj}]
    By \cref{lemma:exploitCost}, the event $\Jof{\Kj} \le \costBound / 2$ occurs with probability at least $1 - \delta / 8T^2$. Similarly to \cref{lemma:lqrSwitch,lemma:GradEstBound}, we implicitly assume that this event holds, which does not break i.i.d assumptions inside the epoch and implies that $\Kji \in \Jdomain$ for all $1 \le \ind1 \le \mj$.
    Now, notice that 
    $
        \EE\brk[s]{\Kji \mid \Kj}
        =
        \Kj
        .
    $
    Since $\Kj \in \Jdomain$ and $\rj \le \pRadius$, we can invoke the local smoothness of $\Jof{\cdot}$ (see \cref{lemma:fazel}) to get that
    \begin{align*}
        \EE\brk[s]{\Jof{\Kji} \mid \Kj}
        &\le
        \Jof{\Kj}
        +
        \nabla \Jof{\Kj}\tran \EE\brk[s]{\Kji - \Kj \mid \Kj}
        +
        \frac12 \pSmooth \EE\brk[s]{\norm{\Kji - \Kj}^2 \mid \Kj}
        \\
        &=
        \Jof{\Kj}
        +
        \frac12 \pSmooth \rj^2
        .
    \end{align*}
    We thus have that
    \begin{align*}
        \sum_{\ind1 = 1}^{\mj} \Jof{K_{\ind0, \ind1}} - \Jof{\Kj}
        \le
        \frac12 \pSmooth \rj^2 \mj
        +
        \sum_{\ind1 = 1}^{\mj} \Jof{K_{\ind0, \ind1}}
        -
        \EE\brk[s]{\Jof{\Kji} \mid \Kj}
        .
    \end{align*}
    The remaining term is a sum of zero-mean i.i.d.\ random variables that are bounded by $\costBound$. We use Hoeffding's inequality and a union bound to get that with probability at least $1 - \delta / 4 T$
    \begin{align*}
        \sum_{\ind1 = 1}^{\mj} \Jof{K_{\ind0, \ind1}} - \Jof{\Kj}
        \le
        \frac12 \pSmooth \rj^2 \mj
        +
        \costBound \sqrt{\frac12 \mj  \log \frac{8 T}{\delta}}
        ,
    \end{align*}
    and plugging in the value of $\pSmooth$ from \cref{lemma:fazel} concludes the proof.
\end{proof}

\begin{ack}
We thank Nadav Merlis for numerous helpful discussions.
This work was partially supported by the Israeli Science Foundation (ISF) grant 2549/19, by the Len Blavatnik and the Blavatnik Family foundation, and by the Yandex Initiative in Machine Learning. 
\end{ack}

\bibliography{bibliography}

\clearpage
\onecolumn
\appendix

\section{Reducing Gaussian Noise to Bounded Noise} 
\label{sec:gaussian}
In this section we relax the bounded noise assumption, $\norm{w_t} \le \noiseBound$, and replace it with the following tail assumption. For $\delta > 0, T \ge 1$, suppose there exists
$
    \cropNoiseSet
    \subseteq
    \RR[\dx]
$
such that:
\begin{enumerate}
    \item
    $
    \PP{w_t \in \cropNoiseSet}
    \ge
    1 - \delta / T
    $
    for all $1 \le t \le T$;
    \item $\EE \brk[s]{w_t \indEvent{w_t \in \cropNoiseSet}} = 0.$
\end{enumerate}
The first assumption is a standard implication of any tail assumption. The second assumption implies that we can crop the noise while keeping it zero-mean. While not entirely trivial, this can be guaranteed for any continuous noise distributions.
We note that $\cropNoiseSet$ is a theoretical construct, and is not a direct input to \cref{alg:main}. Indirectly, we use $\cropNoiseSet$ to calculate the parameters
\begin{align*}
    \cropNoiseBound
    =
    \max_{w \in \cropNoiseSet} \norm{w}
    ,
    \quad
    \cropNoiseMinStd^2
    =
    \min_{\norm{x} = 1} \EE\brk[s]{(w_t\tran x)^2} - \sqrt{\delta \EE\brk[s]{(w_t\tran x)^4} / T}
    ,
\end{align*}
which will serve as replacements for $\noiseBound, \noiseMinStd$ in our bounded noise formulation. In practice, our results hold if for the chosen parameters $\delta, \cropNoiseBound, \cropNoiseMinStd$, there exists a set $\cropNoiseSet$ satisfying the above.
Our main findings for unbounded noise are summarized in the following meta-result.

\begin{theorem}
\label{thm:unboundedNoiseMeta}
    Suppose $\delta \in (0, 1)$ is such that $\cropNoiseMinStd > 0$.
    If we run \cref{alg:main} with the parameters as in \cref{thm:main} and $\noiseBound,\noiseMinStd$ that satisfy $\noiseBound \ge \cropNoiseBound$ and $0 < \noiseMinStd \le \cropNoiseMinStd$. , then the regret bound of \cref{thm:main} holds with probability at least $1 - 2 \delta$
\end{theorem}
\begin{proof}
    Consider the LQR problem where the noise terms $w_t$ are replaced with
    $
        \cropNoise = w_t \indEvent{w_t \in \cropNoiseSet}
        ,
    $    
     and let $\cropCost, \cropJof{\cdot}$ be the corresponding instantaneous and infinite horizon costs. Notice that by our assumptions, $\cropNoise$ are indeed zero-mean, i.i.d, and satisfy 
     $
        \norm{\cropNoise}
        \le
        \cropNoiseBound
        \le
        \noiseBound
    $
    and
     \begin{align*}
         \min_{\norm{x} = 1} \EE\brk[s]{(\cropNoise\tran x)^2}
         &=
         \min_{\norm{x} = 1} \EE\brk[s]{(w_t\tran x)^2} - \EE\brk[s]{\indEvent{w_t \notin \cropNoiseSet} (w_t\tran x)^2}
         \\
         &\ge
         \min_{\norm{x} = 1} \EE\brk[s]{(w_t\tran x)^2}
         -
         \sqrt{\PP{w_t \notin \cropNoiseSet} \EE\brk[s]{(w_t\tran x)^4}}
         \\
         &\ge
         \min_{\norm{x} = 1} \EE\brk[s]{(w_t\tran x)^2}
         -
         \sqrt{\delta \EE\brk[s]{(w_t\tran x)^4} / T}
         =
         \cropNoiseMinStd^2
         \ge
         \noiseMinStd^2
         >
         0
         ,
     \end{align*}
     where the second transition used the Cauchy–Schwarz inequality.
     We thus have that
     $
     \sum_{t=1}^{T} (\cropCost - \cropJstar)
     $
     is bounded as in \cref{thm:main} with probability at least $1 - \delta$. Next, since $\EE w_t w_t\tran \succeq \EE \cropNoise \cropNoise\tran$, we have that that $\cropJof{\cdot}$ is optimistic with respect to $\Jof{\cdot}$, i.e., $\cropJof{K} \le \Jof{K}$ for all $K$, which implies that $\cropJstar \le \Jstar$. 
     Finally, using a union bound on the tail assumption, we have that $w_t = \cropNoise$ for all $1 \le t \le T$ with probability at least $1 - \delta$. On this event, \cref{alg:main} is not aware that the noise is cropped and we thus have that $c_t = \cropCost$ for all $1 \le t \le T$. We conclude that with probability at least $1 - \delta$
     \begin{align*}
        \regret
        =
        \sum_{t=1}^{T} (c_t - \Jstar)
        \le
        \sum_{t=1}^{T} (\cropCost - \cropJstar)
        ,
     \end{align*}
     and using another union bound concludes the proof.
\end{proof}

\paragraph{Application to Gaussian noise.}
We specialize \cref{thm:unboundedNoiseMeta} to the case where $w_t \sim \gaussDist{0}{\noiseCov}$, are zero-mean Gaussian random vectors with positive definite covariance $\noiseCov \in \RR[\dx \times \dx]$. The following result demonstrates how to run \cref{alg:main} given upper and lower bounds on the covariance eigenvalues.
\begin{proposition}
\label{prop:mainGaussianNoise}
    Let $\delta \in (0, 1/3)$. Suppose we run \cref{alg:main} with parameters as in \cref{thm:main} and $\noiseBound,\noiseMinStd$ that satisfy
    \begin{align*}
        \noiseBound
        \ge
        \sqrt{5 \dx \maxEigOf{\noiseCov} \log 
        \frac{T}{\delta}}
        ,
        \quad
        \noiseMinStd^2
        \le
        \minEigOf{\noiseCov}(1 - \sqrt{3\delta / T})
        ,
    \end{align*}
    where $\minEigOf{\noiseCov}, \maxEigOf{\noiseCov}$ are the minimal and maximal eigenvalues of $\noiseCov$.
    Then the regret bound of \cref{thm:main} holds with probability at least $1 - 2\delta$.
\end{proposition}

\begin{proof}
    We show that 
    $
        \cropNoiseSet
        =
        \brk[c]{w \mid \norm{\noiseCov^{-1/2} w} \le 
        \sqrt{5 \dx \log (T/\delta)}}
    $
    satisfies the desired assumptions.
    First, by \cref{lemma:noiseBound} we indeed have that
    $
        \PP{w_t \in \cropNoiseSet}
        \ge 
        1 - \delta / T
        .
    $
    Next, denote $x_t = \noiseCov^{-1/2} w_t$ and notice that $x_t \sim \gaussDist{0}{I}$. We thus have that
    \begin{align*}
        \EE \brk[s]{w_t \indEvent{w_t \in \cropNoiseSet}}
        =
        \noiseCov^{1/2}\EE \brk[s]{\noiseCov^{-1/2} w_t \indEvent{w_t \in \cropNoiseSet}}
        =
        \noiseCov^{1/2}\EE \brk[s]*{x_t \indEvent{\norm{x_t}^2 \le 5 \dx \log \frac{\delta}{T}}}
        =
        0
        ,
    \end{align*}
    where the last transition follows from a symmetry argument. We conclude that $\cropNoiseSet$ satisfies our assumptions. We show that $\noiseBound \ge \cropNoiseBound$ and $0 \le \noiseMinStd \le \cropNoiseMinStd$, which then concludes the proof by invoking \cref{thm:unboundedNoiseMeta}. First, we have that for any $w \in \cropNoiseSet$
    \begin{align*}
        \norm{w}
        \le
        \norm{\noiseCov^{1/2}}\norm{\noiseCov^{-1/2} w}
        \le
        \sqrt{5 \dx \norm{\noiseCov} \log \frac{T}{\delta}}
        \le
        \noiseBound
        ,
    \end{align*}
    and so $\noiseBound \ge \cropNoiseBound$. Finally, notice that for any $x \in \RR[\dx]$ $w_t\tran x$ is a zero-mean Gaussian random variable. Standard moment identities for Gaussian variables then give that
    $
        \EE\brk[s]{(w_t\tran x)^4}
        =
        3\EE\brk[s]{(w_t\tran x)^2}^2
        ,
    $
    and so we have that
    \begin{align*}
        \cropNoiseMinStd^2
        =
        \min_{\norm{x}=1} \EE\brk[s]{(w_t\tran x)^2} (1 - \sqrt{3 \delta / T})
        =
        \minEigOf{\noiseCov} (1 - \sqrt{3 \delta / T})
        \ge
        \noiseMinStd^2
        >
        0
        ,
    \end{align*}
    where the last inequality holds by our choice of $\delta < 1/3$.
\end{proof}

\section{Technical Lemmas}
\label{sec:TechnicalLemmas}

\subsection{Summing the Square Roots of Epoch Lengths}

\begin{lemma}
\label{lemma:mjSum}
    Let $\pExp \in [2/3, 1)$ and define $\mj = \mj[0] \pExp^{-2\ind0}$.
    Suppose $n$ is such that
    $
    \sum_{\ind0 = 0}^{n-2} \mj \le T
    ,
    $
    then we have that
    \begin{align*}
        \sum_{\ind0 = 0}^{n-1} \mj^{1/2} 
        \le
        \frac{22}{\pPL \eta}
        \sqrt{T}
        .
    \end{align*}
\end{lemma}
\begin{proof}
    For ease of notation, denote $\pExp = 1 - (\pPL \eta / 3)$ and notice that for our parameter choice it satisfies $\pExp \in [2/3, 1)$. 
    Now, notice that for $x \ge 1$ we have $x-1 \le \sqrt{x^2 - 1}$ and so we have that
    \begin{align*}
        \frac{\pExp^{-n} - 1}{\pExp^{-1}-1}
        =
        \frac{\pExp^{-1} + 1}{\sqrt{\pExp^{-2}-1}}
        \frac{\pExp^{-n} - 1}{\sqrt{\pExp^{-2}-1}}
        \le
        \frac{\pExp^{-1} + 1}{\pExp^{-1}-1}
        \sqrt{\frac{\pExp^{-2n} - 1}{\pExp^{-2}-1}}
        \le
        \frac{2}{1-\pExp}
        \sqrt{\frac{\pExp^{-2n} - 1}{\pExp^{-2}-1}}
        =
        \frac{6}{\pPL \eta}
        \sqrt{\frac{\pExp^{-2n} - 1}{\pExp^{-2}-1}}
        .
    \end{align*}
    Noticing that $\mj$ is a geometric sequence we get that
    \begin{align*}
        \sum_{\ind0 = 0}^{n-1} \mj^{1/2}
        =
        \mj[0]^{1/2}\frac{\pExp^{-n} - 1}{\pExp^{-1} - 1}
        \le
        \frac{6}{\pPL \eta} \sqrt{
        \mj[0]\frac{\pExp^{-2n} - 1}{\pExp^{-2} - 1}
        }
        =
        \frac{6}{\pPL \eta} \sqrt{\sum_{\ind0 = 0}^{n-1} \mj}
        \le
        \frac{6}{\pPL \eta} \sqrt{(1 + \pExp^{-2})\sum_{\ind0 = 0}^{n-2} \mj}
        \le
        \frac{22}{\pPL \eta}
        \sqrt{T}
        ,
    \end{align*}
    where the last transition also used the fact $\pExp^{-2} \le 9/4$.
\end{proof}

\subsection{Randomized Smoothing} \label{sec:zero-ord-proofs}

\begin{proof}[of \cref{lemma:smoothFuncProperties}]
    The first part follows from Stokes' theorem. See Lemma 1 in \cite{flaxman2005online} for details.
    For the second part, notice that
    $
        \nabla \smoothCostOf{x}
        =
        \nabla \EE[B] \costOf{x + r B} 
        =
        \EE[B] \nabla\costOf{x + r B} 
        .
    $
    We can thus use Jensen's inequality to get that
    \begin{align*}
        \norm{\nabla \smoothCostOf{x} - \nabla\costOf{x}}
        =
        \norm{\EE[B]\brk[s]{\nabla\costOf{x + r B} - \nabla \costOf{x}}}
        \le
        \EE[B]\norm{{\nabla\costOf{x + r B} - \nabla \costOf{x}}}
        \le
        \pSmooth r \EE[B] \norm{{B}}
        \le
        \pSmooth r
        ,
    \end{align*}
    where the third transition also used the smoothness (gradient Lipschitz) property of $\cost$, and the last transition used the fact that $B$ is in the unit ball.
\end{proof}

\subsection{Details of \cref{lemma:fazel}}

We review how \cref{lemma:fazel} is derived from \citet{fazel2018global}. For the rest of this section all Lemmas will refer to ones in \cite{fazel2018global}.

The first part of the statement is immediate from their Lemma 13.
Next, notice that
\begin{align*}
    \frac{\minEigOf{Q} \minEigOf{\noiseCov}}{4 \Jof{K} \norm{\Bstar} (\norm{\Astar + \Bstar K} + 1)}
    \ge
    \frac{\costMatLower \noiseMinStd^2}{4\costBound \systemBound 2 \kappa}
    =
    \frac{1}{8\systemBound\kappa^3}
    =
    \pRadius
    .
\end{align*}
We thus have that $K, K'$ satisfy the condition of Lemma 16 and so we get that
\begin{align*}
    \norm{\steadyCov[K] - \steadyCov[K']}
    \le
    4 \brk*{\frac{\Jof{K}}{\minEigOf{Q}}}^2 \frac{\norm{\Bstar}(\norm{\Astar + \Bstar K} + 1)}{\minEigOf{\noiseCov}}
    \norm{K - K'}
    \le
    \frac{4 \costBound^2 \systemBound^2 2 \kappa}{\costMatLower^2 \noiseMinStd^2}
    \norm{K - K'}
    =
    \frac{8 \costBound \systemBound \kappa^3}{\costMatLower}
    \norm{K - K'}
    ,
\end{align*}
thus concluding the second part. 
Next, define
\begin{align*}
    \mathcal{T}_K(X)
    =
    \sum_{t=0}^{\infty} (\Astar + \Bstar K)^t X \brk[s]{(\Astar + \Bstar K)\tran}^t
    ,
    \qquad
    \mathcal{F}_K(X)
    =
    (\Astar + \Bstar K) X (\Astar + \Bstar K)\tran
    ,
\end{align*}
which are linear operators on symmetric matrices.
By Lemma 17 we have that
\begin{align*}
    \norm{\mathcal{T}_K}
    \le
    \frac{\Jof{K}}{\minEigOf{\noiseCov}\minEigOf{Q}}
    \le
    \frac{\costBound}{\noiseMinStd^2\costMatLower} = \kappa^2
    ,
\end{align*}
and by Lemma 19 we have that
\begin{align*}
    \norm{\mathcal{F}_K - \mathcal{F}_{K'}}
    &\le
    2 \norm{\Astar + \Bstar K} \norm{\Bstar} \norm{K - K'}
    +
    \norm{\Bstar}^2 \norm{K - K'}^2
    \\
    &\le
    \brk[s]*{2 \kappa \systemBound + \frac{\systemBound^2}{8 \systemBound \kappa^3}} \norm{K - K'}
    \\
    &\le
    3 \systemBound \kappa \norm{K - K'}
    .
\end{align*}
Now, continuing from the middle of the proof of Lemma 27 we get that
\begin{align*}
    \norm{\costToGoOf[K] - \costToGoOf[K']}
    &\le
    2 \norm{\mathcal{T}_K}^2 \norm{\mathcal{F}_K - \mathcal{F}_{K'}} \norm{Q + {K'}\tran R K'}
    +
    \norm{\mathcal{T}_K}\norm{K\tran R K - {K'}\tran R K'}
    \\
    &\le
    \brk[s]*{
        6 \systemBound \kappa^5 (1 + \norm{K'}^2)
        +
        \kappa^2 (\norm{K} + \norm{K'})
    }\norm{K - K'}
    \\
    &\le
    \systemBound \kappa^5 \brk[s]*{
        6 + 6 \kappa^2 + 12 \pRadius \kappa + 6 \pRadius^2
        +
        2 + 8 \pRadius^2
    }\norm{K - K'}
    \\
    &\le
    16 \systemBound \kappa^7 \norm{K - K'}
    ,
\end{align*}
where the last step used the fact that $\pRadius \le 1 / 8$.
Next, notice that 
\begin{align*}
    \tr{\noiseCov}
    \le
    \tr{\costToGoOf[K_0] \noiseCov} / \costMatLower
    \le
    \Jof{K_0} / \costMatLower
    \le
    \costBound / 4 \costMatLower
    ,
\end{align*}
and thus the fourth property (Lipschitz) follows as
\begin{align*}
    \abs{\Jof{K} - \Jof{K'}}
    =
    \abs{\tr{(\costToGoOf[K] - \costToGoOf[K'])\noiseCov}}
    \le
    \norm{\costToGoOf[K] - \costToGoOf[K']} \tr{\noiseCov}
    \le
    \frac{4 \systemBound \costBound \kappa^7}{\costMatLower} \norm{K-K'}_F
    .
\end{align*}
Next, the fifth statement (Smoothness) follows the ideas of Lemma 28. Concretely, recall that $\nabla \Jof{K} = 2 E_K \steadyCov[K]$ where $E_K = R K + \Bstar\tran \costToGoOf[K] (\Astar + \Bstar K)$. Notice that
\begin{align*}
    &\norm{\steadyCov[K']} 
    \le
    \norm{\steadyCov[K'] - \steadyCov[K]}
    +
    \norm{\steadyCov[K]}
    \le
    2 \costBound / \costMatLower
    ,
    \\
    &\norm{E_K}
    \le
    \norm{R}\norm{K}
    +
    \norm{\Bstar}\norm{\costToGoOf[K]}\norm{\Astar + \Bstar K}
    \le
    \kappa
    +
    \systemBound \kappa \costBound / \noiseMinStd^2
    \le
    2 \systemBound \kappa^3
    ,
\end{align*}
and thus we have that
\begin{align*}
    \norm{\nabla\Jof{K} - \nabla\Jof{K'}}_F
    \le
    \sqrt{\dx} \norm{\nabla\Jof{K} - \nabla\Jof{K'}}
    &=
    2\sqrt{\dx} \norm{
    (E_K - E_{K'}) \steadyCov[K']
    +
    E_K (\steadyCov[K] - \steadyCov[K'])
    }
    \\
    &\le
    2\sqrt{\dx} \brk[s]*{
    \norm{E_K - E_{K'}} \norm{\steadyCov[K']}
    +
    \norm{E_K} \norm{\steadyCov[K] - \steadyCov[K']}}
    \\
    &\le
    2\sqrt{\dx} \brk[s]*{
    \frac{2\costBound}{\costMatLower} \norm{E_K - E_{K'}}
    +
    \frac{16 \costBound \systemBound^2 \kappa^6}{\costMatLower} \norm{K - K'}}
    .
\end{align*}
Now, notice that
$
\norm{\costToGoOf[K']}
\le
\norm{\costToGoOf[K'] - \costToGoOf[K]}
+
\norm{\costToGoOf[K]}
\le
3 \kappa^4
$
and so
\begin{align*}
    \norm{E_K - E_{K'}}
    &
    \le
    \norm{
    R (K - K')
    + 
    \Bstar\tran \costToGoOf[K] (\Astar + \Bstar K)
    -
    \Bstar\tran \costToGoOf[K'] (\Astar + \Bstar K')
    }
    \\
    &
    \le
    \norm{R}\norm{K - K'}
    +
    \norm{\Bstar\tran (\costToGoOf[K] - \costToGoOf[K'])(\Astar + \Bstar K)}
    +
    \norm{\Bstar\tran \costToGoOf[K'] \Bstar (K - K')}
    \\
    &
    \le
    \brk[s]{
        1
        +
        16 \systemBound^2 \kappa^8
        +
        3 \systemBound^2 \kappa^4
    } \norm{K - K'}
    \\
    &
    \le
    20 \systemBound^2 \kappa^8 \norm{K - K'}
    ,
\end{align*}
and combining with the above, yields the desired smoothness condition
\begin{align*}
    \norm{\nabla\Jof{K} - \nabla\Jof{K'}}_F
    \le
    \frac{112 \sqrt{\dx} \costBound \systemBound^2 \kappa^8}{\costMatLower}
    \norm{K - K'}_F
    .
\end{align*}
Finally, the last statement (PL) is immediate from their Lemma 11 as
\begin{align*}
    \frac{\minEigOf{\steadyCov[\Kstar]}}{\minEigOf{\noiseCov}^2 \minEigOf{R}}
    \le
    \frac{\costBound}{4 \noiseMinStd^4 \costMatLower^2} = \frac{4 \costBound}{\kappa^4}
    .
\end{align*}

\section{Concentration inequalities}
\begin{lemma}[Theorem 1.8 of \cite{hayes2005large}]
\label{lemma:vectorAzuma}
    Let $X$ be a very-weak martingale taking values
    in a real-valued euclidean space $E$ such that
    $X_0 = 0$ and for every $i$, $\norm{X_i - X_{i-1}} \le 1$. Then, for every $a > 0$,
    \begin{align*}
        \PP{\norm{X_n} > a}
        \le
        2e^2 e^{-\frac{a^2}{2n}}
        .
    \end{align*}
    Alternatively, for any $\delta \in (0, \frac12 e^{-2})$ we have that with probability at least $1 - \delta$
    \begin{align*}
        \norm{X_n}
        \le
        \sqrt{2n \log\frac{15}{\delta}}
        .
    \end{align*}
\end{lemma}

The following theorem is a variant of the Hanson-Wright inequality \citep{hanson1971bound,wright1973bound} which can be found in \citet{hsu2012tail}.
\begin{theorem} \label{thm:hansonWright}
	Let $x \sim \gaussDist{0}{I}$ be a Gaussian random vector, let $A \in \RR[m \times n]$ and define $\Sigma = A^T A$. Then we have that
	\begin{equation*}
		\PP{\norm{Ax}^2 > \tr{\Sigma} + 2\sqrt{\tr{\Sigma^2}z} + 2 \norm{\Sigma} z}
		\le
		\exp\brk{-z}
		,\qquad
		\text{ for all } z \ge 0.
	\end{equation*}
\end{theorem}

The following lemma is a direct corollary of \cref{thm:hansonWright}.
\begin{lemma} \label{lemma:noiseBound}
	Let $w \sim \gaussDist{0}{\noiseCov}$ be a Gaussian random vector in $\RR[d]$. For any $\delta \in (0, 1/e)$, with probability at least $1 - \delta$ we have that
	\begin{align*}
	\norm{w}
	\le
	\sqrt{5 \tr{\noiseCov} \log \frac{1}{\delta}}
	.
	\end{align*}
\end{lemma}
\begin{proof}
	Consider \cref{thm:hansonWright} with $A = \noiseCov^{1/2}$ and thus $\Sigma = \noiseCov$. Then for $z \ge 1$ we have that
	\begin{equation*}
		\tr{\noiseCov}
		+
		2\sqrt{\tr{\noiseCov^2}z} 
		+
		2 \norm{\noiseCov} z
		\le
		\tr{\noiseCov} z
		+
		2 \tr{\noiseCov} z
		+
		2 \tr{\noiseCov} z
		=
		5 \tr{\noiseCov} z.
	\end{equation*}
	Now, for $x \sim \gaussDist{0}{I}$ we have that $w \overset{d}{=} A x$ (equals in distribution). We thus have that for $z \ge 1$
	\begin{align*}
		\PP{\norm{w} > \sqrt{5 \tr{\noiseCov} z}}
		\le
		\PP{\norm{Ax}^2 > {\tr{\noiseCov} + 2\sqrt{\tr{\noiseCov^2}z} + 2 \norm{\noiseCov} z}}
		\le
		\exp \brk{-z},
	\end{align*}
	and taking $z = \log \frac{1}{\delta} \ge 1$ (since $\delta \in (0, 1 / e)$) concludes the proof.
\end{proof}

\end{document}